\documentclass[twoside]{article}

\usepackage[accepted]{aistats2025}
\usepackage[round]{natbib}

\usepackage{lmodern}   % Uses Latin Modern, a scalable font
\usepackage[utf8]{inputenc} % allow utf-8 input
\usepackage[T1]{fontenc}    % use 8-bit T1 fonts
\usepackage{hyperref}       % hyperlinks
\usepackage{booktabs}       % professional-quality tables
\usepackage{amsfonts}       % blackboard math symbols
\usepackage{nicefrac}       % compact symbols for 1/2, etc.
\usepackage{microtype}      % microtypography
\usepackage{xcolor}         % colors
\usepackage{amsthm}

\usepackage[pdftex]{graphicx}
\usepackage{subcaption}
\usepackage{mathtools}
\usepackage{multirow}
\usepackage{wrapfig}

\newtheorem{propositionappendix}{Proposition}
\usepackage{amsmath,amsfonts,amssymb}
\usepackage{multirow}
\def\xib{{\boldsymbol{\xi}}}
\newcommand{\KL}{D_{\mathrm{KL}}}

\usepackage{xcolor,pifont}
\newcommand*\colourcheck[1]{%
  \expandafter\newcommand\csname #1check\endcsname{\textcolor{#1}{\ding{52}}}%
}
\colourcheck{blue}
\colourcheck{green}
\colourcheck{red}
\newcommand{\xmark}{\ding{55}}%

\begin{document}


\runningauthor{Paula Cordero Encinar, Tobias Schr\"oder, Peter Yatsyshin, Andrew B. Duncan}

\twocolumn[

\aistatstitle{Deep Optimal Sensor Placement for Black Box Stochastic Simulations}

\aistatsauthor{Paula Cordero Encinar${}^1$ \And Tobias Schr\"oder${}^1$ \And Peter Yatsyshin${}^2$ \And Andrew B. Duncan${}^{1,2}$}

\aistatsaddress{${}^{1}$Imperial College London \And  ${}^{2}$ Alan Turing Institute}]

\begin{abstract}
Selecting cost-effective optimal sensor configurations for subsequent inference of parameters in black-box stochastic systems faces significant computational barriers.
We propose a novel and robust approach,  modelling the joint distribution over input parameters and solution with a joint energy-based model, trained on simulation data. Unlike existing simulation-based inference approaches, which must be tied to a specific set of point evaluations, we learn a functional representation of parameters and solution. This is used as a resolution-independent plug-and-play surrogate for the joint distribution, which can be conditioned over any set of points, permitting an efficient approach to sensor placement. 
We demonstrate the validity of our framework on a variety of stochastic problems, showing that our method provides highly informative sensor locations at a lower computational cost compared to conventional approaches.
\end{abstract}

\section{INTRODUCTION}
\begin{table*}[t]
\caption{Comparison of scope of Functional Neural Couplings (ours) and existing methods. 
 }\label{table_comparison}
\vspace{-2mm}
\begin{center}
\footnotesize
\begin{tabular}{ccccc}
\toprule
\multirow{2}{*}{\bf Methods} & Direct PDE &  Neural Operator  & Neural Operator w/ & Functional Neural \\
& solves & surrogate & oracle noise &  Coupling \textbf{(ours)}\\
\midrule
Low-cost evaluation & \textcolor{red}{\xmark}&   \greencheck  & \greencheck & \greencheck\\
Low-cost inversion & \textcolor{red}{\xmark}&  \greencheck  &  \greencheck & \greencheck\\
Supports sensor placement pipeline & \greencheck  &  \greencheck & \greencheck & \greencheck \\
Supports stochastic PDEs &  \greencheck    & \textcolor{red}{\xmark}&  \greencheck  & \greencheck \\
Tractable likelihood & \textcolor{red}{\xmark}  &  \textcolor{red}{\xmark}  & \textcolor{red}{\xmark} & \greencheck \\
\bottomrule
\end{tabular}
\end{center}
\vspace{-2mm}
\end{table*} 

A common challenge across many areas of science and engineering is recovering an unobserved parameter or input from noisy, indirect observations. Examples arise in imaging (e.g. in-painting and up-scaling), weather forecasting and oceanography (flow reconstruction), material design (molecular force-field reconstruction) and medicine (computed tomography).

Such \emph{inverse
problems} are often ill-posed, meaning that there may be several choices of model parameters which are consistent with the observations, or that the output is highly sensitive to errors in the parameter. Many highly relevant inverse problems are also large-scale, where the map from input parameters to output is sufficiently complex to be effectively \emph{black-box}, and its evaluation is highly computationally intensive.
\\\\
An inverse problem can be formulated as the following system of equations
$$
\begin{aligned}
    u &= \mathcal{G}(\kappa)\\
    \mathbf{y} &= \mathcal{O}(u) + \mathbf{\eta},
\end{aligned}
$$
where the \emph{forward problem} $\mathcal{G}$ is an associated mathematical model, mapping an (unobserved) input parameter $\kappa$ to a spatially-varying field $u$, e.g. the solution of a partial differential equation (PDE). In typical settings, we do not observe $u$ completely, but rather make partial observations $\mathbf{y}$ through an \emph{observation operator} $\mathcal{O}$ subject to noise $\mathbf{\eta}$.

For this work we are interested in settings where the relationship between $\kappa$ and $u$ is non-deterministic, so that $\mathcal G$ possesses intrinsic stochasticity.  This challenge arises in settings where the forward model is not a complete representation of reality, and additional (unobserved) noise sources are introduced to characterise external/environmental effects. 
Such forward models are typically formulated as stochastic partial differential equations (SPDEs).

The Bayesian approach to inverse problems \citep{Stuart_2010} is a systematic method for uncertainty quantification of parametric estimates in which the parameter and the observations are viewed as coupled-random variables. By placing a prior for the distribution of the parameter $\kappa$, Bayes' rule yields a posterior distribution for $\kappa$ given the observations, thus providing a probabilistic method to solving inverse problems. 

The quality of the inference will depend on how informative the observations are which is characterised by the observation operator $\mathcal{O}$ which is typically a vector of point-wise evaluations of $u$ at a fixed set of \emph{sensor locations} within the problem domain.  In settings where we have control over sensor placement, we have a strong incentive to choose them optimally, i.e. to be as informative of $\kappa$ as possible.

While optimal sensor placement has been widely studied in the scientific computing literature \citep{bed_95, rainforth2023modern},  established approaches are infeasible for larger-scale problems.
Typical approaches to sensor placement involve a nested Markov Chain Monte Carlo (MCMC) approach to perform Bayesian Experimental Design (BED). 
In the setting of a PDE-governed inverse problem, each MCMC step necessitates at least one simulation of the underlying PDE, meaning that the process quickly becomes computationally infeasible \citep{alexanderian2021optimal}.

A natural strategy would be to perform optimal sensor placement over a surrogate model \citep{gramacy2020surrogates}, learnt from forward model evaluations.
However, such approaches are not applicable if the parameter is functional or if the forward operator is intrinsically stochastic. In the latter case, the likelihood becomes intractable due to the introduction of auxiliary noise variables, and one must resort to expensive pseudo-marginal MCMC methods \citep{pseudo_likelihood}.

To overcome these limitations we propose \emph{Functional Neural Couplings}, a novel probabilistic model of the relationship between inputs and outputs of a stochastic forward problem. Functional Neural Couplings provide accurate likelihood evaluations for $\kappa$ given noisy observations of $u$, which enables low-cost solution of both the forward and inverse-problem, despite requiring a single training phase.  Specifically, we make the following contributions:

\begin{enumerate}
    \item We introduce \emph{Functional Neural Couplings} (FNC), a probabilistic model of the joint distribution of a functional parameter $\kappa$ and the functional solution of the forward problem $u$ as a joint Energy-Based Model (EBM). 
    \item We leverage implicit neural representations (INR) to express functional data as low-dimensional latent codes. This introduces a new way to extend the scope of energy-based models to functional data on diverse geometries and makes our approach resolution independent, i.e. the model can be queried from functions evaluated at arbitrary evaluation points. 
    \item Finally, we use the Functional Neural Coupling framework to introduce the first computationally feasible method for deep optimal sensor placement for stochastic inverse problems with black-box noise.
\end{enumerate}

\subsection{Related Work}
\paragraph{Neural Operators}
Neural Operators are a class of deep learning architectures designed to learn maps between infinite-dimensional function spaces. As data-driven methods, they do not require any knowledge of the underlying PDE.  Additionally, they are resolution invariant in the sense that they can generate predictions at any resolution. Examples include DeepONets and Fourier Neural Operators (FNO) \citep{ kovachki2021neural, li2021fourier}.
A related approach, CORAL \citep{serrano2023operator} uses an encode-process-decode structure to learn functional operators between functions on non-uniform geometries by leveraging INRs as in our framework.

Some works have leveraged Neural Operators for solving inverse problems, e.g. \citet{inverse_molinaro_23}. 
\citet{long2024invertible} adopts a Bayesian approach by introducing an invertible FNO enhanced with a VAE that allows for posterior inference.
However, these methods assume only observational noise and cannot account for stochasticity in the forward operator.

\paragraph{Generative Models On Function Spaces}
Recent work has sought to extend generative models to distributions on function spaces, which can therefore capture stochastic relationships between the input parameter $\kappa$ and $\mathcal{G}(\kappa)$. 
\citet{rahman2022generative} propose a Generative Adversarial Neural Operator for learning function data distributions.
\citet{baldassari2023conditional}, \citet{franzese2023continuoustime}, \citet{pmlr-v206-kerrigan23a}, \citet{lim2023scorebased} and \citet{pidstrigach2023infinitedimensional}
generalise diffusion models to operate directly on a Hilbert space. In contrast to our approach, these methods do not immediately yield a likelihood which can be used for downstream tasks. 
\cite{mishra2022pi} propose a Variational Autoencoder approach which employs a function space decoder. Similar approaches include Neural Processes \citep{garnelo2018conditional}, Energy-Based Processes \citep{yang2020energy} and Functional EBMs \citep{pmlr-v206-lim23a}.

\paragraph{Inverse Problems and Sensor Placement}
Methods that use black-box simulators for the statistical inference of underlying finite dimensional parameters are well-established and typically referred to as likelihood-free or simulation-based inference methods \citep{cranmer2020frontier}. In recent years, neural networks have been employed to learn synthetic likelihood surrogates with normalising flows \citep{papamakarios2019sequential} and energy-based models \citep{glaser2022maximum} as in our work. 
However, these approaches neither apply to functional data nor address sensor placement. 

The problem of optimising sensor locations is typically approached as a combinatorial optimisation task by discretising the domain into a finite grid \citep{jsan9030031, 10.1115/1.1410929, Andersson_environmental}, while we are able to select sensor placement sites across the entire problem domain.

Sensor placement in the context of inverse PDE problems has mainly been addressed using traditional MCMC approaches \citep{pseudo_likelihood, alexanderian2021optimal}. The number of simulations of the forward system is orders of magnitude higher in classical MCMC-based approaches than in our approach and thus not suitable for complex situations with expensive simulators such as stochastic PDEs.

\section{PROBLEM SETTING}
We consider a stochastic forward problem $\kappa\rightarrow u = \mathcal{G}(\kappa, \omega)$, where the random variable $\omega$ captures the intrinsic stochasticity. The function-valued parameter $\kappa\in \mathcal A$ is defined on a domain $\Omega \subseteq \mathbb R^{d_x}$. For a forward model which involves the solution of a (S)PDE, the parameter $\kappa$ could be an initial condition, boundary condition, spatially-varying coefficient field or forcing term, but our framework is not restricted to these settings, or SPDEs more broadly.

At training time we have access to a black-box stochastic simulator of the forward dynamics to obtain a solution $u=\mathcal{G}(\kappa, \omega)$ for any given $\kappa\in \mathcal A$. To reflect real-world settings, the realisation or distribution of the stochastic contribution $\omega$ is assumed to be \emph{unknown} to us at all times.

\paragraph{Objective}

At test time, the objective is to find optimal sensor positions $\xib$ which maximise the information gained on $\kappa\in\mathcal{A}$ through inference based on noisy observations  $\mathbf y = \mathcal{O}(u) + \eta = (u(\xib_j)+\eta_j)_{j=1}^D$ at sensor locations $\{\xib_j\}_{j=1}^D$. Here, the observational noise $\boldsymbol{\eta}$ is a property of the sensor and is different from the intrinsic stochasticity of the system.

\section{FUNCTIONAL NEURAL COUPLINGS}
We propose to approximate the solution operator $\mathcal G$ using a probabilistic model $p_\theta$ that learns the joint distribution of $(\kappa, u)$ such that high likelihood regions of $p_\theta$ correspond to solutions $u = \mathcal G(\kappa)$. 
To do so, using training data from simulations $\{(\kappa_i(\mathbf x_j^i), u_i(\mathbf x_j^i))_{j = 1}^{N_i}\}_{i=1}^M$, $\mathbf x_j^i\in \Omega$, we encode the functions $\kappa_i$ and $u_i$ into finite-dimensional latent codes $\mathbf z_{\kappa, i} \in\mathbb{R}^{d_{z_\kappa}}, \mathbf z_{u,i}\in\mathbb{R}^{d_{z_u}}$ by employing implicit neural representations \citep{dupont2022coin}. Critically, no evaluation grid has to be defined for this step.
On the finite-dimensional representation space, we learn the joint distribution of the latent codes using a joint energy-based model $p_\theta(\mathbf z_\kappa, \mathbf z_u) \propto \exp(-E_\theta(\mathbf z_\kappa, \mathbf z_u))\,$. The neural network architecture of the joint energy-based model is described in the Appendix. Our training workflow is visualised in Figure \ref{fig:diagram}. 
\begin{figure}[t]
 \hspace{1cm} \centering
  \includegraphics[width=0.55\textwidth]{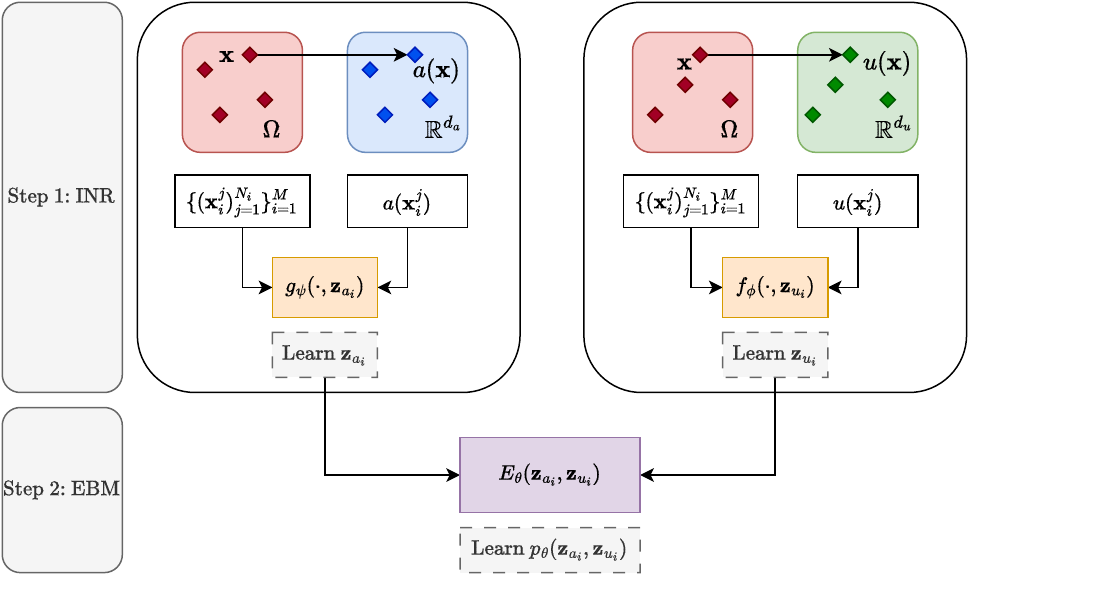}
  \caption{Workflow for the training of the  INR and joint energy-based model. Layout based on \citet{serrano2023operator}.}
  \label{fig:diagram}
\end{figure}
 
\subsection{Learning Implicit Neural Representations}
Following COIN++ \citep{dupont2022coin}, we compress the functional data points $(\kappa_i)_{i=1}^M$ and $(u_i)_{i=1}^M$ into implicit neural representations defined by $\kappa_i(\cdot) = g_\psi(\cdot, \mathbf z_{\kappa, i})$, $\mathbf z_{\kappa,i} \in\mathbb{R}^{d_{z_\kappa}}$, and $u_i(\cdot) = f_\phi(\cdot, \mathbf z_{u,i})$, $\mathbf z_{u,i}\in\mathbb{R}^{d_{z_u}}$. 
We train the neural representations by minimising the mean square error between the function value and the neural network prediction at the evaluation points. 
Each layer of $g_\psi$ and $f_\phi$ takes the form of a SIREN layer $\sin(\omega_0(W\mathbf h + \mathbf b + \boldsymbol\beta))$ \citep{sitzmann2019siren}. While the weights and biases $W, \mathbf b$ of each layer are shared among all data points, the shifts $\boldsymbol \beta$ depend differentiably on each data point latent code $\mathbf z$ and are trained individually for each functional data point.
That is, the training consists of a double optimisation loop, the inner loop updates the modulations $\mathbf z_{\kappa,i}, \mathbf z_{u,i}$ while the outer loop updates the shared parameters, $\psi= \{ W_{\psi}, \mathbf b_{\psi}\}, \phi= \{ W_{\phi}, \mathbf b_{\phi}\}$. 
This produces highly compressed latent representations $(\mathbf z_{\kappa,i}, \mathbf z_{u,i})_{i=1}^M$, which can be decoded efficiently by applying the maps $g_\psi$ and $f_\phi$ , respectively. 
Thus, the INR component is crucial to effectively learn functional relationships, without being limited to a particular discretisation of the domain.

Due to the generalisation capabilities of the INR, we train it on datasets of increasing size until the reconstruction error stabilises on a validation dataset. It is important to remark that this is a practical choice to optimise one of the most computationally expensive steps — learning the INR — without sacrificing performance.
The latent codes $\mathbf{z}$ are then computed for the entire dataset, keeping the shared parameters of the INR fixed.

\subsection{Functional Neural Couplings with Joint Energy-Based Models}
Energy-based models (EBMs) \citep{ebm_lecunn} are unnormalised statistical models of the form $\exp(-E_\theta)$, where the energy-function $E_\theta$ is typically modelled with a scalar-valued neural network. 
We assume lossless compression in the INR and learn the joint distribution of $\kappa$ and $u=\mathcal G(\kappa)$ as a joint energy-based model over tuples $\mathbf z_{i} = (\mathbf z_{\kappa, i}, \mathbf z_{u, i})$. 
Since unnormalised models are not amenable to optimisation with maximum likelihood estimation, we explore training the model with contrastive divergence \citep{Hinton06}, score-based methods \citep{JMLR:v6:hyvarinen05a, vincent_pascal, song2020improved} and energy discrepancy \citep{schroder2023energy}, achieving the best results with energy discrepancy. Note that the joint energy-based model component plays a central role in modelling the stochastic relationship between parameters and solutions.

\section{SENSOR PLACEMENT WITH BAYESIAN EXPERIMENTAL DESIGN}

The goal is to determine optimal sparse sensor placement positions $\xib = \{\xib_1, \xib_2, \dots, \xib_D\} \subset \Omega$ for the inference of $(\kappa, u= \mathcal G(\kappa))$ based on $\mathbf y = u(\xib)+\boldsymbol\eta$. 
We assess the utility of a sensor  position by calculating the expected information gain over the prior as measured by relative entropy, i.e. we use the utility function $U(\xib) := \mathbb E_{p(\mathbf y\vert\xib)}\KL{(p(\mathbf z_\kappa, \mathbf z_u \vert \mathbf y, {\xib})}\,\Vert\,{p_\theta(\mathbf z_\kappa, \mathbf z_u)})$, which can be evaluated at any point $\xib$ thanks to the INR component.  
Using Bayes theorem, the utility can be rewritten as 
\begin{equation}\label{eq:utility_1}
        U(\xib)= \mathbb{E}_{p_\theta(\mathbf{z}_\kappa, \mathbf{z}_u)p( \mathbf{y}|\mathbf{z}_\kappa, \mathbf{z}_u,\xib)}\left[\log\frac{p(\mathbf{y}|\mathbf{z}_\kappa, \mathbf{z}_u, \xib)}{p(\mathbf{y}|\xib)} \right].
\end{equation}
Since $p(\mathbf{y}|\xib)$ is unknown, a naive Monte Carlo estimation provides a nested Monte Carlo estimator \citep{Rainforth2018OnNM}, which approximates the inner and outer integrals of the utility but suffers from a slow rate of convergence. Therefore, following on \citet{pmlr-v108-foster20a}, we use the prior contrastive estimation (PCE) bound given by
\begin{align*}\label{eq:pce_bound}
    \Hat U_{\text{PCE}}(\xib)  = \mathbb{E}\left[\log \frac{p(\mathbf y|\mathbf z_{\kappa, 0}, \mathbf z_{u, 0}, \mathbf \xib)}{\frac{1}{L+1}\sum_{l=0}^L  p(\mathbf y|\mathbf z_{\kappa,l}, \mathbf z_{u,l},\xib)}\right]
\end{align*}
where the expectation is computed over the distribution $p_\theta(\mathbf z_{\kappa,0}, \mathbf z_{u,0})p(\mathbf y|\mathbf z_{\kappa,0}, \mathbf z_{u,0},\xib)p_\theta(\mathbf z_{\kappa,1:L}, \mathbf z_{u,1:L})$.
This alternative estimator significantly speeds up training and is a valid lower bound of the utility which becomes tight as $L\to\infty$ \citep{pmlr-v108-foster20a}. Further details are given in the Appendix.

The actual selection of optimal locations can be conducted sequentially, which leads to Bayesian adaptive design (BAD) \citep{rainforth2023modern} for sensor placement.
This framework iteratively places  sensors by incorporating observations of the solution from previously identified optimal locations, using them to guide the selection of the best sensor positions in subsequent steps.
In settings where multiple measurements can be taken in parallel, each step of the optimisation can select a batch of optimal sensor placement positions instead of just one.
More specifically, considering the sequence of locations and outcomes up to step  $t$ of the experiment, denoted by $\{\xib_1, \dots, \xib_{t-1}\}$ and $\{\mathbf{y}_1,\dots, \mathbf{y}_{t-1}\}$, respectively, we aim to maximise the utility given the history $h_{t-1} = \{(\xib_k, \mathbf{y}_k)\}_{k=1}^{t-1}$,
\begin{equation*}
    {U}({\xib_t}|h_{t-1})=\mathbb{E}\left[ \log\frac{p(\mathbf{y}|\mathbf{z}_\kappa,\mathbf{z}_u, \xib_t, h_{t-1})}{p(\mathbf{y}|\xib_t, h_{t-1})}\right]\,,
\end{equation*}
where $h_0 = \emptyset$ and the expectation is over  $p(\mathbf{z}_\kappa,\mathbf{z}_u|h_{t-1}) p (\mathbf{y}_t|\mathbf{z}_\kappa,\mathbf{z}_u, \xib_t, h_{t-1})$. This can also be interpreted as the utility in Eq. \eqref{eq:utility_1} with an updated prior and likelihood. However, in most cases $\mathbf{y}_t$ is independent of $h_{t-1}$ given $(\mathbf{z}_\kappa, \mathbf{z}_u, \xib_t)$, thus only the prior needs to be updated.

\subsection{Inference from sparse observations}
At inference time, we have noisy observations of the system at the selected optimal locations $\mathcal{D} = \{(\xib_{j}, \mathbf y_{j})\}_{j=1}^D$ with $ \mathbf y_j = u(\xib_j)+\eta_j$, where the observational noise $\eta_j\sim\mathcal{N}(0,\sigma^2)$ is to be distinguished from the random component inherent in the forward model which is captured by the energy-based model.
The posterior distribution of the latent representations $(\mathbf{z}_\kappa, \mathbf{z}_u)$ conditioned on the observed data is given by
\begin{align*} 
p(\mathbf{z}_\kappa,\mathbf{z}_u|\mathcal{D})&\propto p(\mathcal{D}|\mathbf{z}_\kappa, \mathbf{z}_u) p_{\theta}(\mathbf{z}_\kappa, \mathbf{z}_u) \notag\\
&= \prod_{j = 1}^n p(\xib_{j}, \mathbf y_{j}|\mathbf{z}_\kappa, \mathbf{z}_u) p_{\theta}(\mathbf{z}_\kappa, \mathbf{z}_u).
\end{align*}
The noise assumption together with the underlying forward model, results in $p(\xib_j,\mathbf y_j|\mathbf{z}_\kappa, \mathbf{z}_u)= \mathcal{N}(\mathbf y_j;f_{\phi}(\xib_j, \mathbf{z}_{u}), \sigma^2)$.
The desired parameter-solution pair $(\kappa, u=\mathcal G(\kappa))$ can now be sampled from the posterior using stochastic gradient Langevin dynamics \citep{sgld_11}.

\section{NUMERICAL EXPERIMENTS}
In this section, we test the performance of the proposed Functional Neural Coupling framework for optimal sensor placement in stochastic versions of a 1D boundary value problem, the 2D Darcy flow problem and the 2D Navier-Stokes equation. 
We compare our method with approaches based on Fourier neural operators. 
Our training data consists of $M$ pairs of parameters and their corresponding solutions evaluated at $N_i$ point observations that can be different across the $M$ function realisations, that is, $\{(\kappa_i(\mathbf x_j^i), u_i(\mathbf x_j^i))_{j = 1}^{N_i}\}_{i=1}^M$ with $\mathbf x_j^i\in\Omega$.
While the method can handle functional parameters through the INR encoding, we assume for simplicity in the first presented example that $\kappa$ is parametrised by a real-valued vector of finite dimension. 
We emphasise that direct evaluations of $\mathcal{G}$ are only required to train the INRs and the energy-based model. 
Once trained, no further evaluations of $\mathcal{G}$ are required, and the INRs and EBM are used exclusively for optimal sensor placement and inference, drastically reducing the computational burden when $\mathcal{G}$ is a computationally intensive simulator model.  
Additional experimental details (dataset generation, implementation, run times) and further numerical experiments are included in the Appendix.

\paragraph{Benchmarks} 
We compare our method for optimal sensor placement against using a Fourier Neural Operator (FNO)  surrogate for the forward model \citep{li2021fourier}. 
In this case, the posterior distribution of the parameter given observations of the solution is obtained following Section 5.5 of \citet{li2021fourier}, which is subsequently used to perform sensor placement.
As neural operators can only learn deterministic maps, they fail to incorporate the effect that a spatio-temporal external random signal has on the system described.  To provide a gold standard  baseline, we additionally consider the FNO with oracle noise from \citet{salvi2022neural} which assumes that the driving noise of the stochastic model, that we previously denoted as $\omega$, is observed and taken as input to the model. 
The FNO with oracle noise is an unattainable baseline in practice as it has the benefit of having full knowledge of the noise driving the forward model. 
In contrast, our approach does not rely on observing the driving noise $\omega$, which is typically the case in real-world situations where the system's intrinsic stochasticity cannot be directly measured.
For each surrogate method (ours, FNO and FNO with oracle noise), we determine optimal sensor locations using both adaptive and batch non-adaptive approaches. 
We compare the efficiency of optimally selected positions with the use of points from a Quasi Monte Carlo sequence \citep{quasi_random_harald_1992} within the domain. 

Note that we do not compare our approach with traditional methods for sensor placement, as our objective is to develop an efficient, lightweight framework for Bayesian sensor placement and classical approaches are computationally prohibitive or infeasible for the type of problems considered, often requiring run times on the order of several days.
Table \ref{table_comparison} offers a clear comparison of the key advantages of our approach over other methods, highlighting its distinctive characteristics.

\subsection{Boundary Value Problems In 1D}
Consider the boundary value problem (BVP) on the interval $[-1,1]$ given by the non-linear PDE
\begin{align*}\label{eq:conditions}
    &u''(x)-u^2(x)u'(x)=f(x), \quad u(-1) = X_a,\;\; u(1) = X_b,\nonumber\\
    &f(x) = -\pi^2\sin(\pi x)-\pi\cos(\pi x)\sin^2(\pi x),\nonumber\\
     &X_a\sim \mathcal{N}(a, 0.3^2),\;\; X_b\sim \text{Unif}(b-0.3, b+0.4), 
\end{align*}
where $a,b\sim \text{Unif}(-3, 3).$ The training data consists of pairs of parameters $(a,b)$ and their corresponding solution for a realisation of the random variables $X_a$ and $X_b$.
Once trained the INR and joint energy-based models, we carry out a toy sensor placement task, in which we aim to select two optimal locations.
We observe that, using our Functional Neural Coupling surrogate, posterior samples based on information from batch non-adaptive optimal locations match closer the ground truth compared to those from optimal adaptive locations or points of a random sequence (see Figure \ref{fig:bed_1d_non_linear_poisson} for the case with true parameters $a=-2$, $b = 1.5$). 
Table \ref{table_bed_1d_non_linear_poisson} shows the mean and standard deviation of the relative $L^2$ error norm across 100 experiments with different random values for $(a,b)$ for the different methods. The performance of our approach is comparable to that of FNO with oracle noise, which has complete knowledge of the intrinsic stochasticity. In addition,  our method consistently exhibits a lower standard deviation compared to the other approaches.

\begin{table*}[t]
\caption{Relative $L^2$ error norm for the posterior mean of the solution and MSE for the boundary conditions $a, b$ for sensor placement on the BVP. We average over 100 experiments with different random values for $(a,b)$.}
\label{table_bed_1d_non_linear_poisson}
\begin{center}
\footnotesize
\begin{tabular}{ccccc}
\textbf{Method} & \multicolumn{1}{c}{\bf Design points}  &$\Vert \widehat u-u_{\text{tr}}\Vert^2/\Vert u_{\text{tr}}\Vert^2$ & MSE($\hat a$) &  MSE($\hat b$)
\\ \hline \\
\multirow{ 3}{*}{Functional Neural Coupling (Ours)} & Adaptive BED & $0.124\pm0.101$& $0.135\pm 0.099$ & $1.441 \pm 1.003$\\
& Batch non-adaptive BED      &$\mathbf{0.097\pm 0.081}$ & $\mathbf{0.103\pm0.193}$ & $\mathbf{1.074\pm 0.906}$ \\
& Random sequence  &$0.331\pm0.259$ & $0.476\pm 0.888$ & $4.162\pm 3.885$ \\
 \midrule
\multirow{ 3}{*}{\shortstack{FNO surrogate \\ \citep{li2021fourier}}} & Adaptive BED      & $0.137\pm 0.308$ & $ 0.441\pm 1.337$  &   $1.224 \pm 2.084$\\
& Batch non-adaptive BED    & $0.125\pm 0.255$  & $0.428\pm 1.239$ & $1.111\pm 1.785$\\
& Random sequence  & $0.251\pm 0.507$ & $0.484\pm 0.947$ & $ 3.839\pm 7.457$ \\
 \midrule
  \midrule
\multirow{3}{*}{\shortstack{FNO w/ oracle noise surrogate \\ \citep{salvi2022neural}}} & Adaptive BED      &  $0.116\pm 0.250$& $\underline{0.021\pm 0.058}$  & $1.580\pm 2.888$\\
& Batch non-adaptive BED    & $\underline{0.090\pm 0.131}$ & {$0.041 \pm 0.105$} & $\underline{1.046\pm 1.620}$\\
& Random sequence & $0.356 \pm 0.613$  & $0.494 \pm 1.412$ & $8.372\pm 11.856$ \\

\end{tabular}
\end{center}
\end{table*}
\begin{figure}[t]
  \centering
     \begin{subfigure}[t]{\textwidth}
         \includegraphics[width=0.49\textwidth]{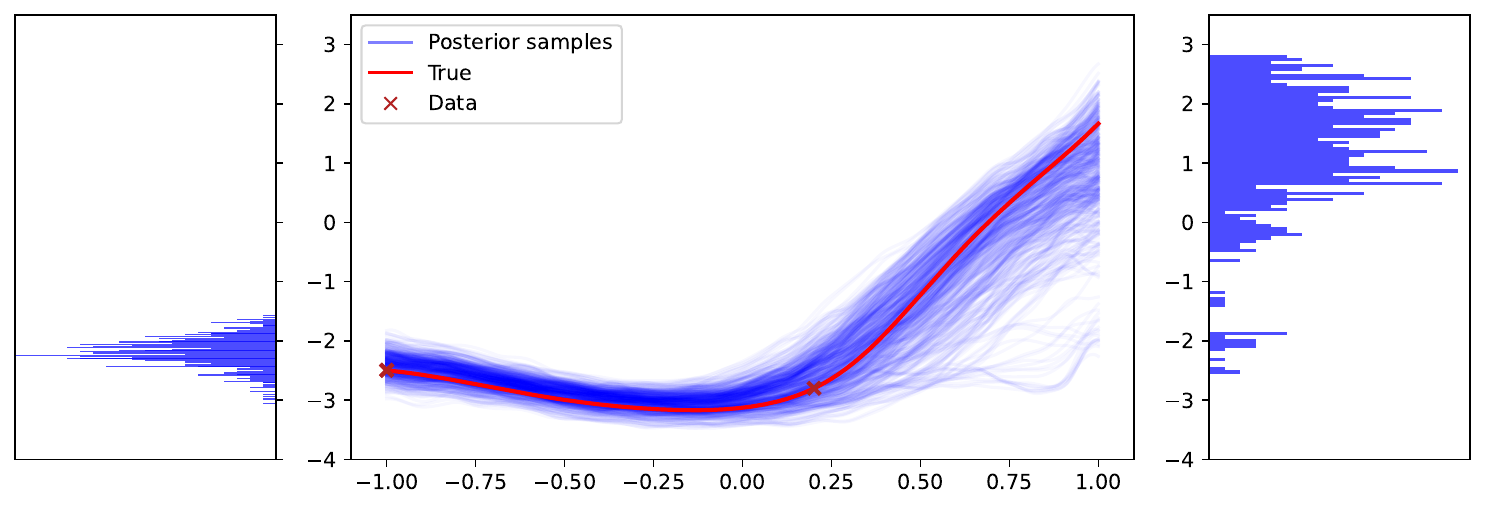}
         \label{fig:bed_1d_optimal_adaptive}
     \end{subfigure}
     \vfill
          \begin{subfigure}[t]{\textwidth}
         \includegraphics[width=0.49\textwidth]{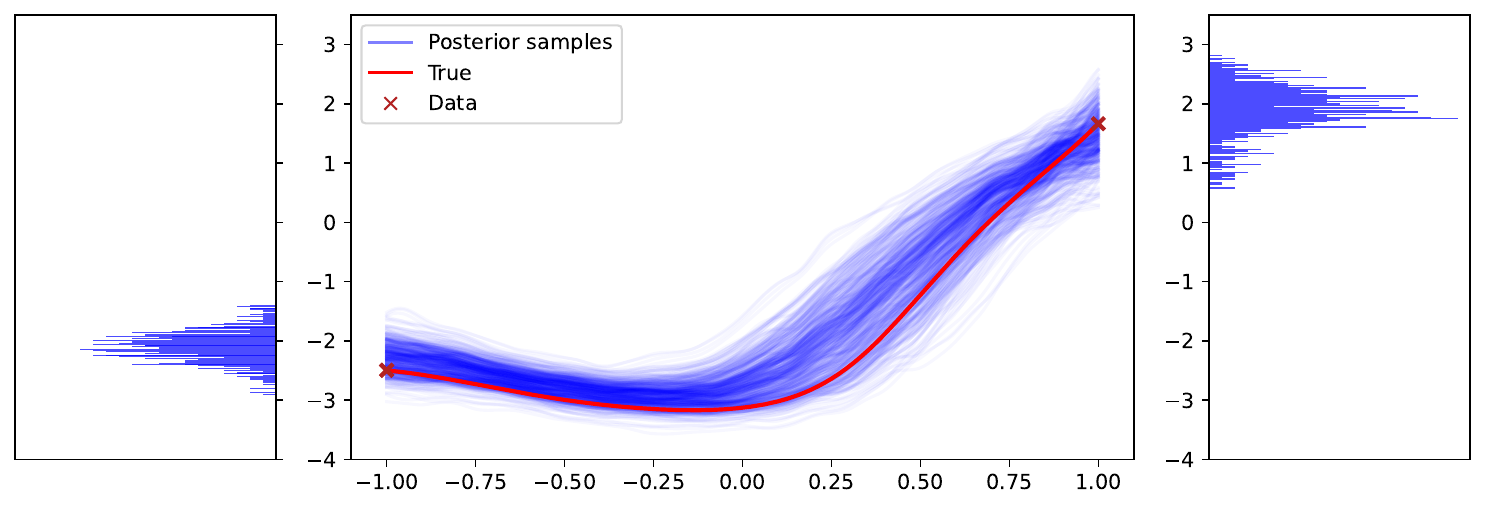}
         \label{fig:bed_1d_optimal_batch}
     \end{subfigure}
     \vfill
         \begin{subfigure}[t]{\textwidth}
         \includegraphics[width=0.49\textwidth]{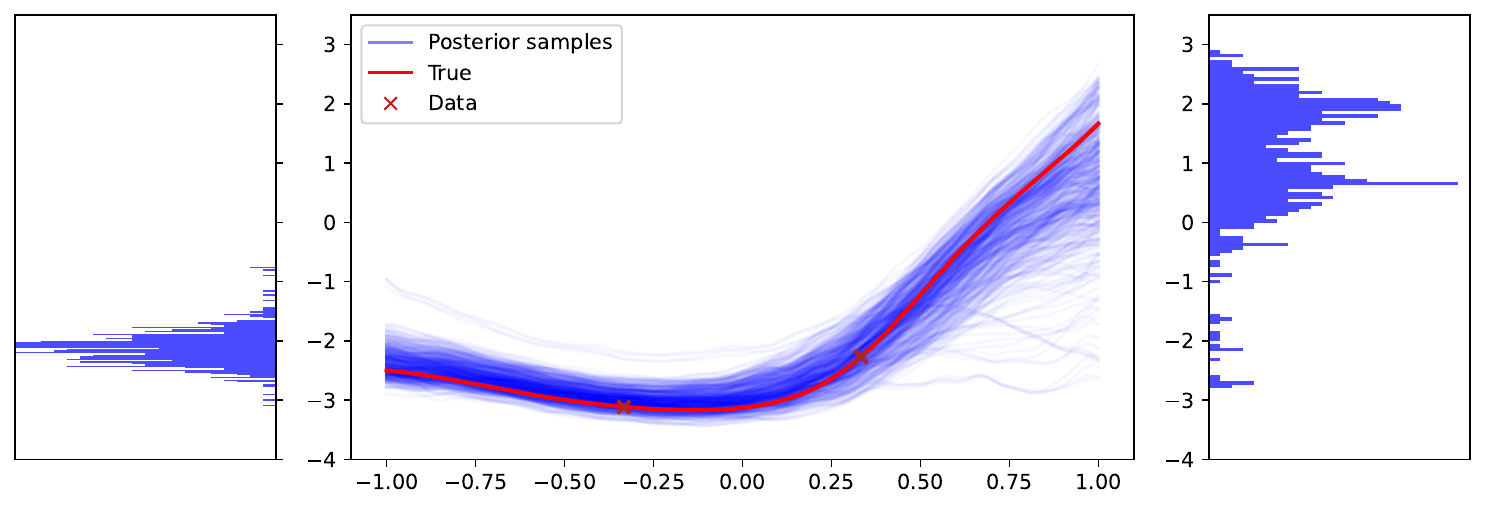} 
        \label{fig:bed_1d_random}
     \end{subfigure}
     \caption{Sensor placement task for the BVP using our Functional Neural Coupling surrogate. Posterior samples based on two optimal design points, adaptive (top) and batch non-adaptive (middle) versus two points from a random sequence (bottom). True values $a=-2$, $b=1.5$.}  
  \label{fig:bed_1d_non_linear_poisson}
  \vspace{-0.2cm}
 \end{figure}

 \subsection{Steady-State Diffusion In 2D}\label{sec:bed_diffusion} 
We consider learning the diffusion coefficient $\kappa$ of the stochastic 2D Darcy flow equation defined by the PDE $ -\nabla\cdot\big(\kappa(\mathbf x)\nabla u(\mathbf x)\big) = f(\mathbf x) + \alpha \omega$ with domain $\mathbf x \in \Omega= [0,1]^2$ and Dirichlet boundary conditions  $u|_{\partial \Omega} = 0$. 
The force term is of the form $f(x) = 0.5$, $\omega$ is space white noise  and the diffusion coefficient $\kappa$ is generated as the push-forward of a Gaussian random field. 

We perform a sensor placement task in which, based on five initial observations (intentionally chosen in non informative places) of a randomly chosen solution, we maximise the estimated utility function to find optimal locations for fifteen additional measurement sites.

Figures \ref{fig:bed_posterior_solutions_darcy} and \ref{fig:bed_posterior_coefficients_darcy} present the inference results of our Functional Neural Coupling surrogate, comparing sensor placement at selected locations using the optimal adaptive method, the batch non-adaptive approach, and points from a Quasi Monte Carlo sequence within the domain. Qualitatively, the results obtained with adaptive sensor placements better match the ground truth both for the solution and the diffusion coefficient.
Table \ref{table_bed} shows that the relative $L^2$ error norms between the ground truth and the posterior mean for the solution and the diffusion coefficient averaged over 50 sensor placement experiments. 
It is important to note that no errors are reported for the inference based on Quasi-Monte Carlo points, as Quasi-Monte Carlo sequences are deterministic. We observe that, across different types of sensor locations, our Functional Neural Coupling surrogate outperforms the FNO surrogate, which was trained with the same data, and achieves performance close to that of the FNO with oracle noise surrogate, which has complete knowledge of the system.

\begin{table*}[t]
\caption{Relative $L^2$ error norm for the posterior mean of the diffusion coefficient and the solution for the sensor placement experiment on the Darcy flow equation. We take the average over 50 sensor placement loops.}
\label{table_bed}
\begin{center}
\footnotesize
\begin{tabular}{cccc} 
\multicolumn{1}{c}{\bf Method} & \multicolumn{1}{c}{\bf Design points} & $\Vert \log\widehat{\kappa}-\log\kappa_{\text{tr}}\Vert^2/\Vert \log\kappa_{\text{tr}}\Vert^2$   &$\Vert \widehat u-u_{\text{tr}}\Vert^2/\Vert u_{\text{tr}}\Vert^2$ 
\\ \hline \\
\multirow{ 3}{*}{Functional Neural Coupling (Ours)} & Adaptive BED & $\mathbf{0.234\pm 0.078}$    & $\mathbf{0.102 \pm 0.083}$ \\
& Batch non-adaptive BED  & $ 0.337\pm 0.091$& $0.192\pm 0.087$\\
& Quasi-Monte Carlo sequence  & $0.711$    & $0.328$\\
 \midrule
\multirow{ 3}{*}{\shortstack{FNO surrogate \\ \citep{li2021fourier}}} & Adaptive BED      &  $0.306\pm 0.130$ &  $0.117\pm 0.106$\\
& Batch non-adaptive BED    & $0.551\pm 0.172$ & $0.220\pm 0.118$\\
& Quasi-Monte Carlo sequence  &  $1.182$& $0.379$\\
\midrule
\midrule
\multirow{3}{*}{\shortstack{FNO w/ oracle noise surrogate \\ \citep{salvi2022neural}}} & Adaptive BED      & $\underline{0.155 \pm 0.101}$ & $\underline{0.093\pm 0.089}$\\
& Batch non-adaptive BED    & $0.291\pm 0.089$ & $0.124\pm 0.110$\\
& Quasi-Monte Carlo sequence  & $0.459$& $0.255$\\

\end{tabular}
\end{center}
\end{table*}

\begin{figure}[t]
  \centering
  \includegraphics[width=0.4\textwidth]{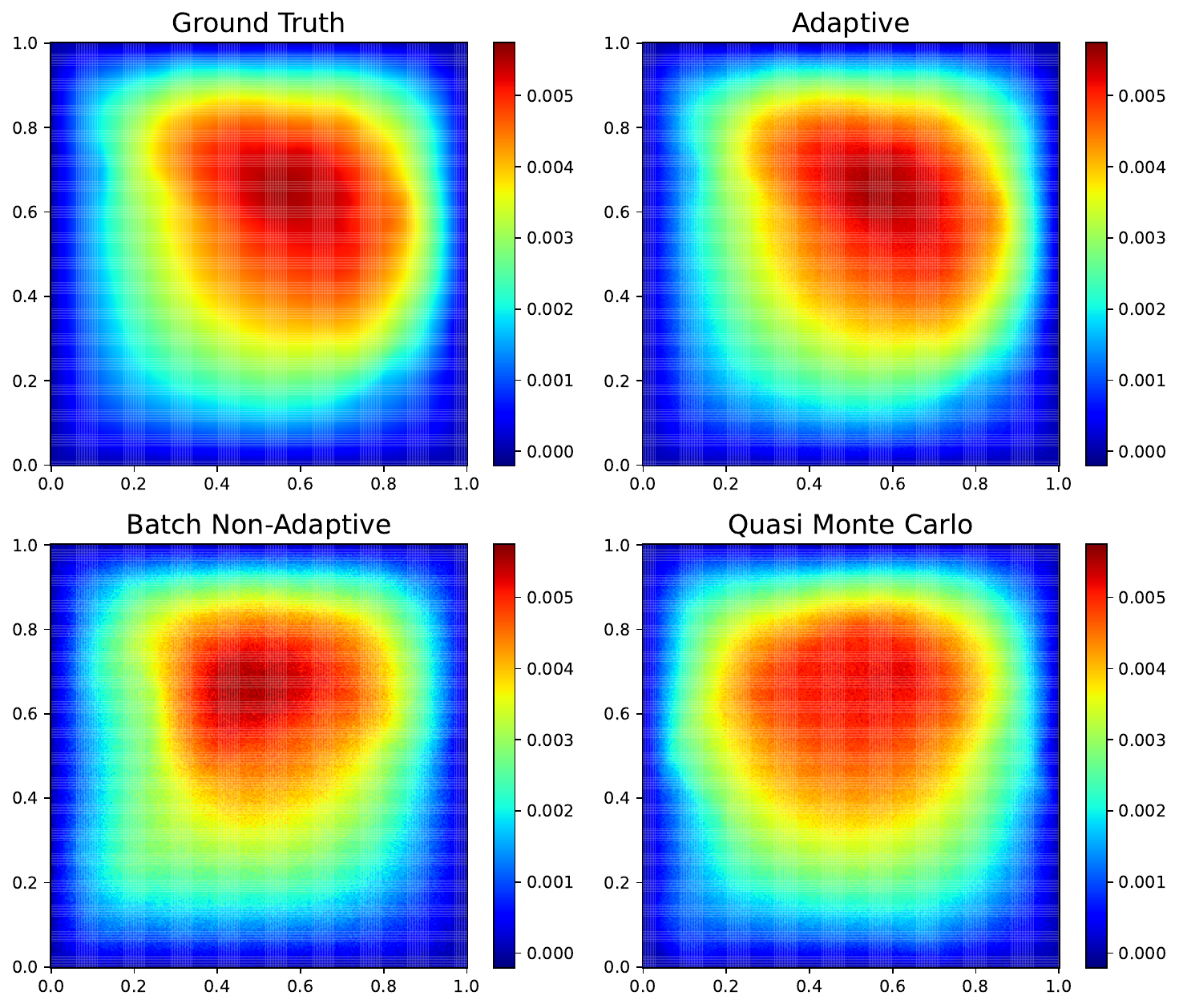}
  \caption{Top left: Observed solution. Posterior mean solutions from observations at adaptively chosen locations (top right), batch non-adaptively selected locations (bottom left) and Quasi Monte Carlo points (bottom right).}
  \label{fig:bed_posterior_solutions_darcy}
  \vspace{-0.2cm}
\end{figure}
\begin{figure}[t]
  \centering
  \includegraphics[width=0.4\textwidth]{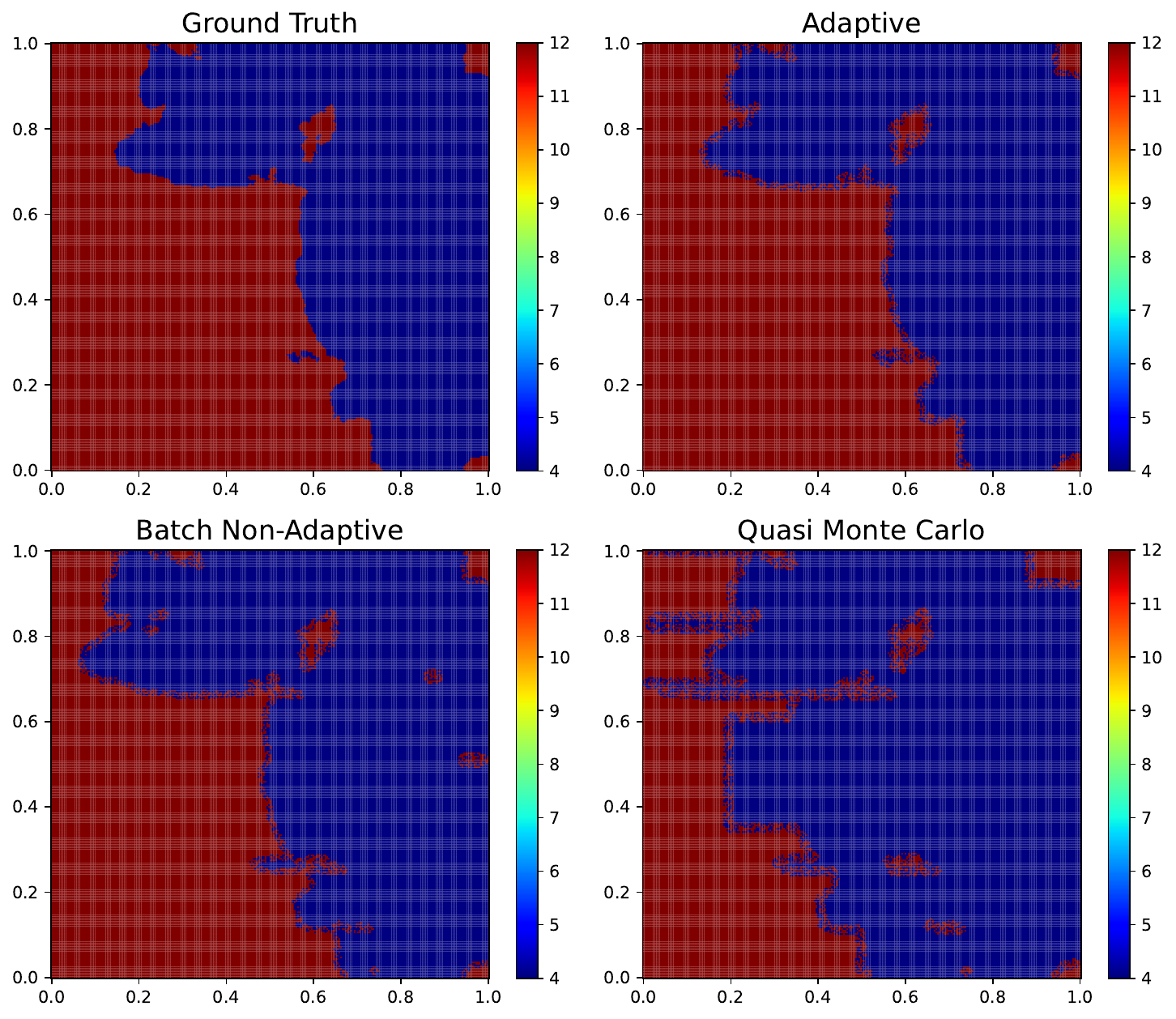}
  \caption{Top left: True diffusion coefficient. Posterior mean coefficients from observations at adaptively chosen locations (top right), batch non-adaptively selected locations (bottom left) and Quasi Monte Carlo points (bottom right).}
  \label{fig:bed_posterior_coefficients_darcy}
  \vspace{-0.2cm}
\end{figure}

\subsection{Navier Stokes Equation}
We study a stochastic version of the 2D Navier Stokes equation for a viscous, incompressible fluid in vorticity form on the unit torus
\begin{gather*}
    \partial_t w(x,t) + u(x,t)\cdot \nabla w(x,t) = \nu \Delta w(x,t) + f(x) + \alpha\varepsilon,\\
    \nabla\cdot u(x,t) = 0,\quad
    w(x,0) = w_0(x),
\end{gather*}
where $x\in(0,1)^2,\ t\in (0, T]$, $u$ is the velocity field, $w=\nabla\times u$ is the vorticity, $\nu\in\mathbb{R}_+$ is the viscosity coefficient, which is set to $10^{-4}$, $f$ is the deterministic forcing function, which takes the form $f(x) = 0.1[\sin(2\pi(x_1+x_2)) + \cos(2\pi(x_1+x_2))]$, and $\varepsilon$ is the stochastic forcing given by $\varepsilon=\Dot{W}$ for $W$ being a Q-Wiener process which is coloured in space. The initial condition $w_0(x)$ is generated according to a Gaussian random field with periodic boundary conditions. 
Following \citet{salvi2022neural}, we solve the previous equation for each realisation of the Q-Wiener process using a pseudo-spectral solver and a time step of size $10^{-3}$. 
The training data is generated by evaluating the solution at time points $t = 1, 2, 3$ on a $16 \times 16$ spatial mesh.
We learn a Functional Neural Coupling between the initial vorticity $\omega_0$ and the vorticity at times $t = 1, 2, 3$.

The sensor placement task  consists in finding optimal locations for fifteen additional measurement sites based on five initial observations of a randomly chosen solution. Inference is performed using observations of the vorticity at $t=1, 2, 3$ in the chosen locations.
As illustrated in Figure \ref{fig:bed_navier_stokes}, inference results using adaptive sensor placement locations closely match the ground truth.
Table \ref{table_bed_navier} shows that our Functional Neural Coupling approach outperforms the FNO surrogate and performs close to the FNO with oracle noise which is unattainable in practice.

\begin{table*}[t]
\caption{Relative $L^2$ errors for the posterior mean of the initial vorticity $w_0$ and the vorticity at $t=1, 2, 3$, $\underline{w}_t$, for the sensor placement experiment on the Navier Stokes equation. We take the average over 20 sensor placement experiments.}
\label{table_bed_navier}
\begin{center}
\footnotesize
\begin{tabular}{cccc} 
\multicolumn{1}{c}{\bf Method} & \multicolumn{1}{c}{\bf Design points} &  $\Vert \widehat{w}_0- w_{\text{tr},0}\Vert^2/\Vert w_{\text{tr},0}\Vert^2$   &$\Vert \widehat{\underline{w}}_{t}-\underline{w}_{\text{tr}, t}\Vert^2/\Vert \underline{w}_{\text{tr}, t} \Vert^2$
\\ \hline \\
\multirow{ 3}{*}{Functional Neural Coupling (Ours)} & Adaptive BED & $\mathbf{0.293\pm0.077}$ & $\mathbf{0.175\pm 0.091}$\\
& Batch non-adaptive BED  & $0.321\pm0.083$ & $0.239\pm0.090$\\
& Quasi-Monte Carlo sequence  &  $ 0.578$& $0.422$\\
 \midrule
\multirow{ 3}{*}{\shortstack{FNO surrogate \\ \citep{li2021fourier}}}  & Adaptive BED      & $0.382\pm0.067$ & $0.242\pm 0.095$\\
& Batch non-adaptive BED    & $0.454\pm0.092$ & $0.288\pm 0.089$\\
& Quasi-Monte Carlo sequence  & $0.652$ & $0.576$ \\
\midrule
\midrule
\multirow{3}{*}{\shortstack{FNO w/ oracle noise surrogate \\ \citep{salvi2022neural}}} & Adaptive BED      & $\underline{0.221\pm0.065}$  & $\underline{0.103\pm0.079}$\\
& Batch non-adaptive BED    & $0.301\pm0.080$ & $0.169\pm0.083$\\
& Quasi-Monte Carlo sequence  &$0.454$ & $0.332$\\
\end{tabular}
\end{center}
\end{table*}

\begin{figure}[t]
  \centering
     \begin{subfigure}[t]{\textwidth}
         \includegraphics[width=0.49\textwidth]{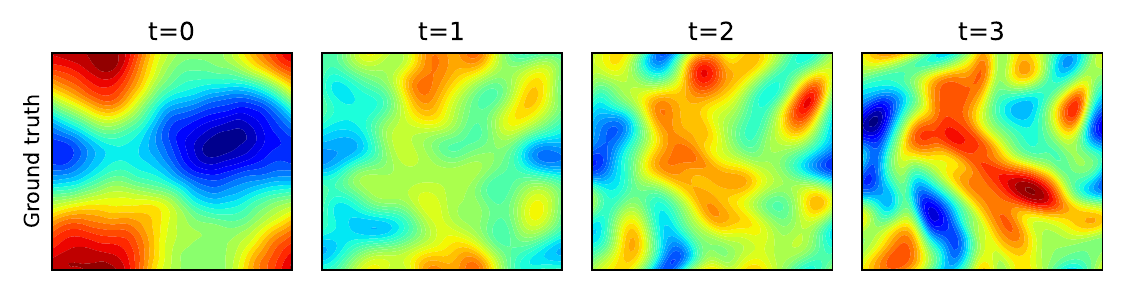}
         \label{fig:bed_ns_ground_truth}
     \end{subfigure}
     \vfill
          \begin{subfigure}[t]{\textwidth}
         \includegraphics[width=0.49\textwidth]{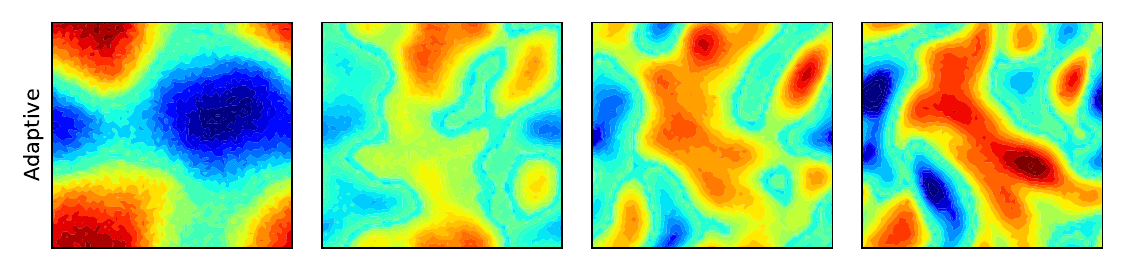}
         \label{fig:bed_ns_adaptive}
     \end{subfigure}
     \caption{Top: ground truth solution. Bottom: posterior mean based on observations at adaptively chosen sensor locations using our Functional Neural Coupling surrogate.}  
  \label{fig:bed_navier_stokes}
  \vspace{-0.2cm}
 \end{figure}

\section{DISCUSSION AND LIMITATIONS}
\label{sec:discussion_limitations} 
In this paper we introduce Functional Neural Couplings, a novel approach  to learning resolution-invariant surrogate models for stochastic functional operators by modelling the joint distribution of operator input and output with an energy-based model without assuming any knowledge about the intrinsic driving noise.
The combination of probabilistic generative models and the implicit neural representation unlocks a unique set of properties of our method: it is resolution independent (meaning that even if trained on a lower resolution it can be directly evaluated on a higher resolution), it captures the effects of  stochasticity present in the forward model (as is the case for SPDEs), and it outputs an unnormalised joint distribution which allows inference and downstream tasks such as sensor placement through sampling from the posterior distribution. 
To the best of our knowledge, our work is the first that makes use of generative models in sensor placement of inverse problems for SPDEs avoiding costly traditional MCMC-based methods and providing a tractable likelihood model.
The numerical experiments demonstrate that our approach outperforms the FNO surrogate in the sensor placement task, as the latter are unsuitable for stochastic forward problems.
While our approach does not require full knowledge of the driving noise, it is still able to achieve performance comparable to the FNO having full access to the driving noise, thus yielding an accurate and practically useful surrogate. 
We also show improved accuracy when performing inference with a principled sensor placement strategy using BED tools, compared to using points from a Quasi-Monte Carlo sequence. In both cases, our Functional Neural Coupling serves as a surrogate for the joint distribution.

While Functional Neural Couplings are a promising method for probabilistic surrogate modelling of expensive stochastic systems, they make the assumption that the density of parameter-solution pairs is positive everywhere. This may pose challenges to our method when the data distribution is sparsely distributed on $\mathcal A\times \mathcal U$, for example, in fully deterministic settings or when only few parameter choices lead to stable behaviour of the system. 
In particular, the training of the energy-based model may fail when the implicit neural representation is chosen too high-dimensional, thus we may need to sacrifice the accuracy with which the functions are described through the finite dimensional latent variables to allow for stable training. 
In addition, the energy based model will not necessarily  generalise to parts of the state space not covered by the prior from which $a\in \mathcal A$ was sampled. Thus, the specific training data needs to be chosen judiciously. Alternatively, an adaptive data-generation approach as proposed in \citet{papamakarios2019sequential} might be required, particularly when the simulator is expensive.

For future work, we are interested in exploring improved sample efficient methods for the modelling and training of energy-based models for functional data and studying sequential strategies that use the observation data more effectively for fine-tuning the base energy-based model.

\subsubsection*{Acknowledgements}
PCE is supported by EPSRC through the Modern Statistics and Statistical Machine Learning (StatML) CDT programme, grant no. EP/S023151/1. TS is supported by the EPSRC-DTP scholarship partially funded by the Department of Mathematics, Imperial College London. 
We thank the anonymous reviewers for their comments.

\bibliographystyle{plainnat}
\bibliography{refs}

\section*{Checklist}

 \begin{enumerate}

 \item For all models and algorithms presented, check if you include:
 \begin{enumerate}
   \item A clear description of the mathematical setting, assumptions, algorithm, and/or model. [Yes]
   \item An analysis of the properties and complexity (time, space, sample size) of any algorithm. [Yes] Further details are provided in the Appendix.
   \item (Optional) Anonymized source code, with specification of all dependencies, including external libraries. [Yes] It is provided in the supplementary material.
 \end{enumerate}

 \item For any theoretical claim, check if you include:
 \begin{enumerate}
   \item Statements of the full set of assumptions of all theoretical results. [Yes] We only include one theoretical result, Proposition 1 (in the Appendix). The proof is provided and the statement of the theorem includes all the necessary assumptions.
   \item Complete proofs of all theoretical results. [Yes]
   \item Clear explanations of any assumptions. [Yes]     
 \end{enumerate}

 \item For all figures and tables that present empirical results, check if you include:
 \begin{enumerate}
   \item The code, data, and instructions needed to reproduce the main experimental results (either in the supplemental material or as a URL). [Yes] It is provided in the supplementary material.
   \item All the training details (e.g., data splits, hyperparameters, how they were chosen). [Yes] All the details are provided in the Appendix.
         \item A clear definition of the specific measure or statistics and error bars (e.g., with respect to the random seed after running experiments multiple times). [Yes]
         \item A description of the computing infrastructure used. (e.g., type of GPUs, internal cluster, or cloud provider). [Yes]  The type of computational resources used to run the experiments is specified in the Appendix.
 \end{enumerate}

 \item If you are using existing assets (e.g., code, data, models) or curating/releasing new assets, check if you include:
 \begin{enumerate}
   \item Citations of the creator If your work uses existing assets. [Yes]  In particular, we reuse code (available on GitHub) from previous work, which is appropriately cited where necessary.
   \item The license information of the assets, if applicable. [Yes]
   \item New assets either in the supplemental material or as a URL, if applicable. [Yes]
   \item Information about consent from data providers/curators. [Not Applicable]
   \item Discussion of sensible content if applicable, e.g., personally identifiable information or offensive content. [Not Applicable]
 \end{enumerate}

 \item If you used crowdsourcing or conducted research with human subjects, check if you include:
 \begin{enumerate}
   \item The full text of instructions given to participants and screenshots. [Not Applicable]
   \item Descriptions of potential participant risks, with links to Institutional Review Board (IRB) approvals if applicable. [Not Applicable]
   \item The estimated hourly wage paid to participants and the total amount spent on participant compensation. [Not Applicable]
 \end{enumerate}

 \end{enumerate}

 \newpage

\appendix
\onecolumn

\section{Methodology: Further Details}
\subsection{Implicit Neural Representations}
Since we are using the latent representations to learn a surrogate model for the operator $\mathcal G$, similar functions should map to similar latent representations. This is guaranteed via the following proposition.
\begin{propositionappendix}
Let $g_{\psi}(\cdot, \cdot):\Omega\times\mathbb{R}^{d_z}\to\mathbb{R}$ be an implicit neural representation. 
If the activation functions of the neural network are continuous and differentiable almost everywhere, then the function $g_{\psi}$ is continuous and differentiable almost everywhere in both variables.
In addition, if the domain $\Omega\subseteq\mathbb{R}^{d_x}$ is bounded and the neural network activation functions are Lipschitz, then $g_{\psi}$ is Lipschitz with respect to the second variable $\mathbf{z}$, i.e. there exists a constant $L$ such that for all $\mathbf{z}, \mathbf{z}'\in \mathbb{R}^{d_z}$
    \begin{equation*}
        \Vert g_{\psi}(\cdot, \mathbf{z})-  g_{\psi}(\cdot, \mathbf{z}')\Vert_{L_2}\leq L \Vert \mathbf z - \mathbf z'\Vert_{\mathbb{R}^{d_z}}.
    \end{equation*}
\end{propositionappendix}
\begin{proof}
Recall that $g_\psi$ is composed of layers of the form $\sin(w_0(W\mathbf{h} + \mathbf{b} + \mathbf{\beta}))$, where $\beta = \beta(\mathbf{z})$ is a hypernetwork itself with continuous and differentiable almost everywhere activation functions, in particular in our implementation we use ReLU activations which satisfy this condition. Therefore, it immediately follows that $g_\psi$ is continuous and differentiable almost everywhere in both variables.

On the other hand, we have that linear functions are Lipschitz and if the activation functions are also Lipschitz by composition we have that $g_\psi(\mathbf{x},\cdot)$ is Lipschitz for every $\mathbf x\in\Omega$, i.e, there exists a constant $\Tilde{L}$ such that for every $\mathbf x\in\Omega$
    \begin{equation*}
        \Vert g_{\psi}(\mathbf x, \mathbf{z})-  g_{\psi}(\mathbf x, \mathbf{z}')\Vert_\mathbb{R}\leq \Tilde{L} \Vert \mathbf z - \mathbf z'\Vert_{\mathbb{R}^{d_z}}.
    \end{equation*}
Putting this together with the fact that the domain $\Omega\in\mathbb{R}^{d_x}$ is bounded, it follows that $g_\psi$ is Lipschitz with respect to the second variable $\mathbf{z}$.
In particular, in our case we use sine and ReLU activation functions which are Lipschitz with constant 1.
\end{proof}
\subsection{Energy-Based Model Training}
Energy-Based Models are a class of parametric unnormalised probabilistic models of the form $\exp(-E_\theta)$ originally inspired by statistical physics.
Despite their flexibility, energy-based models have not seen widespread adoption in machine learning applications due to the challenges involved in training them.
In particular, because the normalising constant is intractable, these models cannot be optimised using maximum likelihood estimation. 
Alternative training methods have been explore such as contrastive divergence \citep{Hinton06}, score-based methods \citep{JMLR:v6:hyvarinen05a, vincent_pascal, song2020improved} and energy discrepancy \citep{schroder2023energy}.

Contrastive divergence \citep{Hinton06} is a method that approximates the gradient of the log-likelihood via short runs of a Markov Chain Monte Carlo (MCMC) process. 
Although using short MCMC runs greatly reduces both the computation per gradient step and the variance of the estimated gradient \citep{pmlr-vR5-carreira-perpinan05a}, it comes at the cost of producing poor approximations of the energy function. 
This issue arises, in part, because contrastive divergence is not the gradient of any objective function \citep{pmlr-v9-sutskever10a, 10.1162/neco.2008.11-07-647} which significantly limits the theoretical understanding of its convergence.

Score-based methods provide an alternative way for training based on minimising the expected squared distance of the score function of the true distribution $\nabla_x \log p_{\text{data}}$ and the score function given by the model $\nabla_x \log p_\theta$, which by definition are independent of the normalising constant \citep{JMLR:v6:hyvarinen05a, vincent_pascal}. 
However, these methods only use gradient information and are therefore short-sighted \citep{yang_song_gradients_data_distribution} as they do not resolve the global characteristics of the distribution when limited data are available.

Energy discrepancy attempts to solve the problems of the two previous methods by proposing a new loss function that compares the data distribution and the energy-based model \citep{schroder2023energy}. This loss is given by
\begin{equation*}
    \mathcal{L}_{t, M, w}(\theta) :=\frac{1}{N}\sum_{i=1}^N \log\left(\frac{w}{M} + \frac{1}{M}\sum_{j=1}^M\exp\left(E_\theta (\mathbf{z}_i) - E_\theta(\mathbf z_{i} + \sqrt{t}\xib_i + \sqrt{t}\xib_{i,j}')\right)\right),
\end{equation*}
where $\xib_i, \xib_{i,j}'\sim \mathcal{N}(0,\mathbf I_{{d_{z_a} + d_{z_u}}})$ are i.i.d. random variables and $t, M, w$ are tunable hyper-parameters. 
Energy discrepancy depends only on the energy function and is independent of the scores and MCMC samples from the energy-based model. 

We applied these three methods to train the energy-based model in our framework, obtaining the best results with the latter. Notably, for the 2D examples, both contrastive divergence and score matching failed to converge.

\subsection{Optimal Bayesian Experimental Design For Sensor Placement}\label{sec:sensor_placement_appendix}
The sensor placement task consists in optimally selecting $\xib$ to maximise the information gain about the parameter and solution of the stochastic PDE. 
In our proposed framework, parameter and solution are approximated by $g_{\psi}(\cdot,\mathbf{z}_{\kappa})$ and $f_{\phi}(\cdot,\mathbf{z}_{u})$, respectively, where $\mathbf{z}_\kappa\in\mathbb{R}^{d_{z_\kappa}}$ and $\mathbf{z}_u\in\mathbb{R}^{d_{z_u}}$ are the associated latent embeddings. 
Mathematically, the utility function for $\xib$ needs to maximise the expected information gain over the prior $p_\theta(\mathbf{z}_\kappa, \mathbf{z}_u)$, as measured by relative entropy. This is equivalent to maximising the expected KL-divergence
\begin{align*}
    U(\xib) &= \mathbb{E}_{p(\mathbf{y}|\xib)}\big[\KL\big({p(\mathbf{z}_\kappa, \mathbf{z}_u|\xib, \mathbf{y})}\,\Vert \,{p_\theta(\mathbf{z}_\kappa, \mathbf{z}_u)}\big)\big] \\
    &= \int d(\mathbf{z}_\kappa, \mathbf{z}_u) \int d\mathbf{y} \big(\log p(\mathbf{z}_\kappa, \mathbf{z}_u|\xib, \mathbf{y}) - \log p_\theta(\mathbf{z}_\kappa, \mathbf{z}_u)\big) p(\mathbf{z}_\kappa, \mathbf{z}_u, \boldsymbol{y}|\xib),
\end{align*}
where the expectation is computed over the predictive distribution of the new (yet unobserved) data $p(\mathbf{y}|\xib)$. 
Applying Bayes theorem, we rewrite the above expression in a form amenable to estimation
\begin{align}\label{eq:utility}
    U(\xib)&=\int d(\mathbf{z}_\kappa, \mathbf{z}_u) \int d\mathbf{y}\left(\log \frac{p(\mathbf{y}|\mathbf{z}_\kappa, \mathbf{z}_u, \xib)p_\theta(\mathbf{z}_\kappa, \mathbf{z}_u)}{p(\mathbf{y}|\xib)} - \log p_\theta(\mathbf{z}_\kappa, \mathbf{z}_u)\right) p(\mathbf{z}_\kappa, \mathbf{z}_u, \boldsymbol{y}|\mathbf{d})\nonumber\\
    &= \int d(\mathbf{z}_\kappa, \mathbf{z}_u) \int d\mathbf{y} \big(\log p(\mathbf{y}|\mathbf{z}_\kappa, \mathbf{z}_u, \xib) - \log p(\mathbf{y}|\xib)\big) p(\mathbf{z}_\kappa, \mathbf{z}_u, \mathbf{y}|\xib)\nonumber\\
    &= \mathbb{E}_{p(\mathbf{z}_\kappa, \mathbf{z}_u, \mathbf{y}|\xib)}\log p(\mathbf{y}|\mathbf{z}_\kappa, \mathbf{z}_u, \xib) - \mathbb{E}_{p(\mathbf{y}|\xib)} \log p(\mathbf{y}|\xib).
\end{align}
\label{ap:boed_ace}

Finding the optimal sensor location $\xib$ is very challenging since the density $p(\mathbf{y}| \xib)$ in Eq. (\ref{eq:utility}) is intractable. A naive Monte Carlo approach will require the use of a nested Monte Carlo estimator which results in high variance and converges relatively slow. 
Moreover, the utility $U(\xib)$ must be computed separately for each location $\xib$, which makes the framework highly inefficient especially in high-dimensional settings, as $U(\xib)$ is fed into an outer optimisation loop to select the optimal sensor position.

To circumvent these inefficiencies amortised variational approaches have been proposed in the literature \citep{Foster2019VariationalBO, pmlr-v108-foster20a, Foster2021DeepAD, kennamer_variational_bed}. By constructing a variational lower bound to $U(\xib)$, we can obtain a unified framework that can be simultaneously optimised with respect to both the variational $\varphi$ and position parameters $\xib$ using stochastic gradient ascent \citep{robbings_monro_1951}. \citet{pmlr-v108-foster20a} proposed the adaptive contrastive estimation (ACE) bound given by
\begin{equation*}
    \Hat U_{\text{ACE}}(\xib, \varphi) = \mathbb{E}\left[\log \frac{p(\mathbf y|\mathbf z_{\kappa, 0}, \mathbf z_{u, 0}, \xib)}{\frac{1}{L+1}\sum_{l=0}^L\frac{p_\theta(\mathbf z_{\kappa,l}, \mathbf z_{u,l}) p(\mathbf y|\mathbf z_{\kappa,l}, \mathbf z_{u,l},\xib)}{q_\varphi(\mathbf z_{\kappa,l}, \mathbf z_{u,l}|\mathbf y)}}\right],
\end{equation*}
where $q_\varphi(\mathbf z_{\kappa,0}, \mathbf{z}_{u,0}|\mathbf y)$ is the inference network which takes as input $\varphi$, $\mathbf y$ and outputs a distribution over $(\mathbf z_{\kappa}, \mathbf{z}_{u})$ and the expectation is taken with respect to $p_\theta(\mathbf z_{\kappa,0}, \mathbf z_{u,0})p(\mathbf y|\mathbf z_{\kappa,0}, \mathbf z_{u,0},\xib)q_\varphi(\mathbf z_{\kappa,1:L}, \mathbf z_{u,1:L}|\mathbf y)$. \citet{pmlr-v108-foster20a} prove that $\Hat U_{\text{ACE}}(\xib, \varphi)\leq U(\xib)$. By replacing the inference network $q_\varphi(\mathbf z_{\kappa}, \mathbf z_{u}|\mathbf y)$ with the prior $p_\theta(\mathbf z_{a}, \mathbf z_{u})$ to generate contrastive samples, we obtain the prior contrastive estimation (PCE) bound
\begin{align*}
    \Hat U_{\text{PCE}}(\xib) & = \mathbb{E}\left[\log \frac{p(\mathbf y|\mathbf z_{a, 0}, \mathbf z_{u, 0}, \mathbf \xib)}{\frac{1}{L+1}\sum_{l=0}^L  p(\mathbf y|\mathbf z_{\kappa,l}, \mathbf z_{u,l},\xib)}\right]\\
    &=  - \mathbb{E}\left[\log \left(1+\sum_{l=1}^L \frac{p(\mathbf y|\mathbf z_{\kappa,l}, \mathbf z_{u,l},\xib)}{p(\mathbf y|\mathbf z_{\kappa,0}, \mathbf z_{u,0},\xib)}\right)\right] + C,
\end{align*}
where the expectation is over $p_\theta(\mathbf z_{\kappa,0}, \mathbf z_{u,0})p(\mathbf y|\mathbf z_{\kappa,0}, \mathbf z_{u,0},\xib)p_\theta(\mathbf z_{\kappa,1:L}, \mathbf z_{u,1:L})$.
This alternative estimator significantly speeds up the training since no variational parameters need to be learnt.

To optimise the PCE bound with respect to $\xib$ we need an unbiased gradient estimator of $\partial \Hat U_{\text{PCE}}/\partial \xib$.
This can be achieved by using a reparametrisation trick, where we introduce a random variable $\varepsilon$ independent of $\xib$ together with a representation of $\mathbf y$ as a deterministic function of $\varepsilon$, that is, $\mathbf y=\mathbf y(\mathbf z_{\kappa,0}, \mathbf z_{u,0}, \xib, \varepsilon)$. Recalling that $\mathbf{y}_i\sim\mathcal{N}(f_{\phi}(\xib,\mathbf{z}_{u_i}), \sigma^2)$ we can write $ \mathbf y = f_\phi(\xib, \mathbf z_{u, 0}) + \sigma \varepsilon$, where $\varepsilon\sim \mathcal N(0,1)$. This allows us to derive the following reparametrised gradient 
\begin{equation*}
    \frac{\partial \Hat U_{\text{PCE}}}{\partial \xib} = \mathbb{E}\left[\frac{\partial g}{\partial \xib} + \frac{\partial g}{\partial \mathbf y}\frac{\partial \mathbf y}{\partial \xib} \right],
\end{equation*}
where the expectation is over $p_\theta(\mathbf z_{\kappa,0}, \mathbf z_{u,0}) p(\varepsilon)p_\theta(\mathbf z_{\kappa,1:L}, \mathbf z_{u,1:L})$ and 
\begin{equation*}
    g(\mathbf y,\mathbf z_{\kappa, 0:L}, \mathbf z_{u, 0:L}, \xib) = - \log \left(1 + \sum_{l=1}^L\frac{p(\mathbf y|\mathbf z_{\kappa,l}, \mathbf z_{u,l},\xib)}{p(\mathbf y|\mathbf z_{\kappa,0}, \mathbf z_{u,0},\xib)}\right).
\end{equation*}
A Monte Carlo approximation of this expectation leads to a much lower variance estimator for the true $\xib$-gradient.
Note that $\mathbb{E}[\partial g/ \partial \xib]$ can be efficiently computed in \texttt{PyTorch} as $\partial \mathbb{E}[g]/\partial \xib$ by detaching $\mathbf{y}$ from the computational graph. On the other hand,
\begin{equation*}
    \frac{\partial g}{\partial \mathbf y} = - \frac{1}{\exp(-g)}\sum_{l=1}^L \frac{p(\mathbf y|\mathbf z_{\kappa,l}, \mathbf z_{u,l},\xib)}{p(\mathbf y|\mathbf z_{\kappa,0}, \mathbf z_{u,0},\xib)} \frac{1}{\sigma^2}\left( f_\phi(\xib, \mathbf{z}_{u,l}) - f_\phi(\xib, \mathbf{z}_{u,0})\right),
\end{equation*}
which can be evaluated using the numerically stabilised logsumexp function,
and 
\begin{equation*}
    \frac{\partial \mathbf y}{\partial \xib} = \frac{\partial f_\phi (\xib, \mathbf z_{u,0})}{\partial \xib}.
\end{equation*}

\section{Numerical Experiment: Stochastic Lotka-Volterra Model}
Another active field of research, where our framework can bring significant benefits is inference on models with intractable likelihoods, such as dynamical systems described by stochastic PDEs. The intractability of the density in such systems arises from the need to marginalise the transition probabilities over all possible trajectories of the stochastic system. 
To illustrate this, take the general It\^{o} SDE
\begin{equation*}
    d X_t = a(X_t)dt + dW_t,\;\; X_0 = x_0, \;\; 0\leq t\leq T,
\end{equation*}
where $W_t$ is a Wiener process and $a:\mathbb{R}\to\mathbb{R}$ is a sufficiently regular non-linear drift coefficient. Given $N$ sparse noisy observations of a trajectory, $\{y_i = X_{t_i} + \eta_i\}_{i=1}^N$, where $\eta_i$ are i.i.d. $\mathcal{N}(0, \sigma^2)$. 
Then, the joint distribution over observations and latent realisations is given by
\begin{align*}
    p(y_1, \dots, y_N, X_{t_1},\dots, X_{t_N}) &= p(X_{t_1},\dots, X_{t_N}) \prod_{i=1}^N p(y_i|X_{t_i})  = p(X_{t_1})\prod_{j=1}^{N-1} p(X_{t_{j+1}}|X_{t_{j}}) \prod_{i=1}^N p(y_i|X_{t_i}),
\end{align*}
where for the second equality we have used that the process $\{W_t\}_{t\geq 0}$ has independent increments. We have that $y_i|X_{t_i}\sim \mathcal{N}(X_{t_i}, \sigma^2)$, however in general the distribution of $X_{t_{i+1}}|X_{t_{i}}$ cannot be expressed analytically.
If the observations occur frequently, so that the time between observations $t_i$ and $t_{i+1}$ is small, then  we can approximate the distribution of $X_{t + \Delta t}|X_{t}$ using a Gaussian approximation,
\begin{equation*}
    X_{t + \Delta t}|X_{t}\sim \mathcal{N}(X_t +a(X_{t})\Delta t, \Delta t).
\end{equation*}
The validity of this approximation breaks down as the step size $\Delta t$ increases, necessitating alternative approaches in scenarios where the observations are infrequent.
Our Functional Neural Couplings method can be used to obtain the distribution over possible trajectories  $\{X_t\}_{0\leq t\leq T}$ from a dataset consisting of point observations of different trajectories, without needing to evaluate each trajectory of the training set at the same time steps. 

To provide a specific example, consider a stochastic version of the Lotka-Volterra model, which describes the joint temporal dynamics of two coexisting populations, predator, $X_{1,t}$, and prey, $X_{2,t}$,
{\footnotesize
\begin{align}\label{eq:lotka_volterra_model}
    dX_{1,t} =& (\theta_1 X_{1,t}-\theta_2 X_{1,t}X_{2,t})dt + \sqrt{\theta_1 X_{1,t}}dW_{t}^{(1)}- \sqrt{\theta_2 X_{1,t}X_{2,t}}dW_{t}^{(2)},\nonumber\\
    dX_{2,t} =& -(\theta_3 X_{2,t}-\theta_2 X_{1,t}X_{2,t})dt - \sqrt{\theta_3 X_{2,t}}dW_{t}^{(3)} + \sqrt{\theta_2 X_{1,t}X_{2,t}}dW_{t}^{(2)},
\end{align}}where $0\leq t\leq 1$, $X_{1, 0} = a$, $X_{2, 0} = b$, $W_t$ is a Wiener process and $\theta = [\theta_1, \theta_2, \theta_3]$ are the model parameters, which for our experiment we keep fixed at $\theta=[5, 0.035, 6]$. 
We train our functional neural coupling surrogate to learn a probability distribution over possible trajectories $\{X_{1,t}, X_{2,t}\}_{0\leq t\leq T}$ from a dataset consisting of point observations of simulated trajectories obtained using an Euler-Maruyama scheme with a sufficiently small time step $\Delta t = 0.01$ to ensure numerical stability.

Given the inference data, $\mathcal{D} = \{(y_{1,t_i}, y_{2,t_i})\}_{i=1}^M$, consisting on $M$ noisy observations of a trajectory, i.e. $y_{j,t_i} = X_{j, t_i} + \eta_i$ where $\eta_i\sim \mathcal{N}(0,\sigma^2)$ and $\sigma=0.2$. The posterior distribution over the latent representation $(\mathbf{z_1}, \mathbf{z_2})$ is 
\begin{align*}
    p(\mathbf{z_1}, \mathbf{z_2}| \mathcal{D}) &\propto \prod_{i = 1}^M p(y_{1,t_i}, y_{2,t_i}|g_\psi(t_i, \mathbf{z_1}), f_\phi(t_i, \mathbf{z_2}))p_{\theta}(\mathbf{z_1}, \mathbf{z_2}) \\
    & = \prod_{i = 1}^M \mathcal{N}(y_{1,t_i}|g_\psi(t_i, \mathbf{z_1}), \sigma^2) \mathcal{N}(y_{2,t_i}|f_\phi(t_i, \mathbf{z_2}), \sigma^2) p_{\theta}(\mathbf{z_1}, \mathbf{z_2}),
\end{align*}
where $g_\psi, f_\phi$ denote the INRs for the prey and predator trajectories, respectively.

We perform Bayesian optimal experimental design to place five sensor/measurement times  based on two initial low-informative observations.
Table \ref{table_bed_lotka_volterra} shows the relative $L^2$ error norm for the posterior means of the prey and predator trajectories when averaged across $100$ different prey-predator trajectories. 
Our approach performs comparably to the FNO surrogate with oracle noise, where complete knowledge of the driving noise is assumed, and consistently outperforms the standard FNO surrogate across all types of design points.

\begin{table*}[!h]
\caption{Relative $L^2$ errors for the posterior means of the prey and predator trajectories for the sensor placement task. We average over $100$ different prey-predator trajectories.}
\label{table_bed_lotka_volterra}
\begin{center}
\footnotesize
\begin{tabular}{cccc} 
\multicolumn{1}{c}{\bf Method} & \multicolumn{1}{c}{\bf Design points} &  $\Vert \widehat{X}_1- X_{\text{tr},1}\Vert^2/\Vert X_{\text{tr},1}\Vert^2$   &$\Vert \widehat X_{2}-X_{\text{tr}, 2}\Vert^2/\Vert X_{\text{tr}, 2} \Vert^2$
\\ \hline \\
\multirow{ 3}{*}{Neural Coupling (Ours)} & Adaptive BED & $\mathbf{0.070\pm 0.066}$ & $\mathbf{0.068\pm 0.059}$\\
& Batch non-adaptive BED & $0.081\pm0.100$ & $0.078\pm0.068
$ \\
& Quasi-Monte Carlo sequence  & $0.101\pm0.090$ &  $0.095\pm0.073$\\
 \midrule
\multirow{ 3}{*}{\shortstack{FNO surrogate \\ \citep{li2021fourier}}}  & Adaptive BED      & $0.092\pm0.104$ & $0.091\pm0.199$\\
& Batch non-adaptive BED    & $0.111\pm0.163$ & $0.123\pm0.187$\\
& Quasi-Monte Carlo sequence  & $0.176\pm0.211$& $0.184\pm0.337$\\
\midrule
\midrule
\multirow{3}{*}{\shortstack{FNO w/ oracle noise surrogate \\ \citep{salvi2022neural}}} & Adaptive BED      & $\underline{0.056\pm 0.050}$ & $\underline{0.045\pm 0.041}$\\
& Batch non-adaptive BED    & $0.061\pm 0.052$  & $0.057\pm 0.055$\\
& Quasi-Monte Carlo sequence  &  $0.089\pm 0.049$ & $0.082\pm 0.060$\\
\end{tabular}
\end{center}
\end{table*}

\section{Experimental details: Datasets}\label{sec:details_datasets}
\subsection{Boundary Value Problem}
The dataset to train our neural coupling surrogate is of the form 
\begin{equation*}
    \left\{\left((a_i, b_i), u_i(x_j^i)\right)_{j=1}^{N_i}\right\}_{i=1}^M,
\end{equation*}
where $a, b\sim\text{Unif}[-3, 3]$, $x_j^i\in\Omega =[-1, 1]$ and $u$ represents the solution corresponding to a realisation of the random variables $X_a, X_b$. The solutions $u_i$ are computed using DOLFINx \citep{BarattaEtal2023}. To train the FNO with oracle noise baseline \citep{salvi2022neural}, we also store the particular realisation of $X_a$ and $X_b$ for each data point.

The dataset consists of 10000 pairs of boundary conditions and solutions. Each solution is evaluated at $N_i=30$ points in the domain $[-1, 1]$.

\subsection{Steady-State Diffusion}
The 2D Darcy flow equation is a second-order linear elliptic equation of the form
\begin{equation*}
    -\nabla\cdot\big(\kappa(\mathbf x)\nabla u(\mathbf x)\big) = f(\mathbf x) + \alpha\omega
\end{equation*}
with domain $\mathbf x \in \Omega= [0,1]^2$ and Dirichlet boundary conditions  $u|_{\partial \Omega} = 0$. The force term is kept fixed $f(x) = 0.5$, $\omega$ is space white noise and the diffusion coefficient is generated according to $\kappa\sim\mu$, where $\mu = \psi_\#\mathcal{N}(0, (-\Delta + 9I)^{-2})$ with zero Neumann boundary conditions on the Laplacian. 
The mapping $\psi:\mathbb{R}\to\mathbb{R}$ takes the value 12 on the positive part of the real line and $3$ on the negative and the push-forward is defined point-wise.
The solution is computed using the finite element method.

The dataset consists of 10000 pairs of diffusion coefficients and associated solutions, evaluated on a uniform $16\times 16$ mesh. The particular realisation of $\omega$ is also stored for benchmark comparisons.

\subsection{Navier Stokes Equation}
We consider a stochastic version of the 2D Navier Stokes equation for a viscous, incompressible fluid in vorticity form on the unit torus
\begin{gather*}
    \partial_t w(x,t) + u(x,t)\cdot \nabla w(x,t) = \nu \Delta w(x,t) + f(x) + \alpha\varepsilon,\\
    \nabla\cdot u(x,t) = 0,\quad
    w(x,0) = w_0(x),
\end{gather*}
where $x\in(0,1)^2,\ t\in (0, T]$, $u$ is the velocity field, $w=\nabla\times u$ is the vorticity, $\nu\in\mathbb{R}_+$ is the viscosity coefficient, which is set to $10^{-4}$, $f$ is the deterministic forcing function, which takes the form $f(x) = 0.1[\sin(2\pi(x_1+x_2)) + \cos(2\pi(x_1+x_2))]$, and $\varepsilon$ is the stochastic forcing given by $\varepsilon=\Dot{W}$ for $W$ being a Q-Wiener process which is coloured in space. 
The initial condition is generated according to $w_0\sim\mathcal{N}(0, 3^{3/2}(-\Delta + 49I)^{-3}))$ with periodic boundary conditions. 
Following \citet{salvi2022neural}, we solve the previous equation for each realisation of the Q-Wiener process using a pseudo-spectral solver, where time is advanced with a Crank-Nicolson update using a time step of size $10^{-3}$.
The SPDE is solved on a $64\times 64$ mesh in space. 

To form the training dataset, we subsample the trajectories $w_t$ by a factor of 4 in space (resulting in a $16 \times 16$ spatial resolution) and only consider the initial vorticity $w_0$ and the vorticity at times $t=1, 2, 3$. The dataset consists of 20000 trajectories.

\subsection{Stochastic Lotka-Volterra Model}
To obtain the training dataset, we simulate trajectories for the prey and predator. 
Specifically, we generate 1000 uniformly distributed initial conditions and simulate 100 stochastic trajectories of the model for each initial condition using an Euler-Maruyama scheme for Eq. \eqref{eq:lotka_volterra_model} with a sufficiently small time step $\Delta t = 0.01$ to ensure numerical stability. Each trajectory is evaluated in $30$ different time steps.

\section{Experimental details: Implementation}
All experiments were conducted on a GPU server consisting of eight Nvidia GTX 3090 Ti GPU cards, 896 GB of memory and 14TB of local on-server data storage. Each GPU has 10496 cores as well as 24 GB of memory.

\subsection{Training}
The training is done in two steps. First, we train the modulated INRs to encode the functional data in a finite dimensional latent space. 
Due to the generalisation capabilities of the INR, we train it on datasets of increasing size until the reconstruction error stabilises on a validation dataset. Figure \ref{fig:INR_analysis_performance} shows the evolution of the relative $L^2$ error for the INR reconstructed function, computed on a validation dataset, along with the training times for different percentages of the total dataset size (provided in section \ref{sec:details_datasets}).
It is important to remark that this is a practical choice to optimise one of the most computationally expensive steps — learning the INR — without sacrificing performance.
In all experiments, we use 300 epochs, a batch size of 16 and for the other hyper-parameters we keep the same values as \citet{dupont2022coin}. 
In subsequent experiments, we use 10\% of the dataset for training the INR, as this provides a balance between performance and computational cost.

Once, the INRs have been fitted, we obtain the latent representations of the functions of interest in the whole training dataset.
As the functional data is encoded using only a few steps of gradient descent (specifically 3 steps, for details see \citet{dupont2022coin}), the resulting standard deviation of the latent representations is very small, falling within the range of $[10^{-3}, 10^{-1}]$. Therefore, these raw latent modulations are not appropriate for subsequent processing. To address this, we standardise the codes by subtracting the mean and dividing by the standard deviation. 
Subsequently, the standardised latent modulations are concatenated in a single vector in cases where we seek for a prior over different functional parameters and solutions. 
The concatenated latent embeddings are then used to train the joint energy-based model. 
The final hyper-parameters for the joint energy-based model training are presented in Table \ref{table-training-details-EBM}. 
In addition, in all experiments we use Adam optimiser with a learning rate of $10^{-3}$ and an exponential scheduler.
Note that the means and standard deviations are stored to renormalise the samples generated by the joint energy-based model, which are then used as modulations for the INR to recover the functions in real space.

\begin{figure}[h]
    \centering

    % First subfigure
    \begin{subfigure}[b]{0.3\textwidth}
        \centering
        \includegraphics[width=\textwidth]{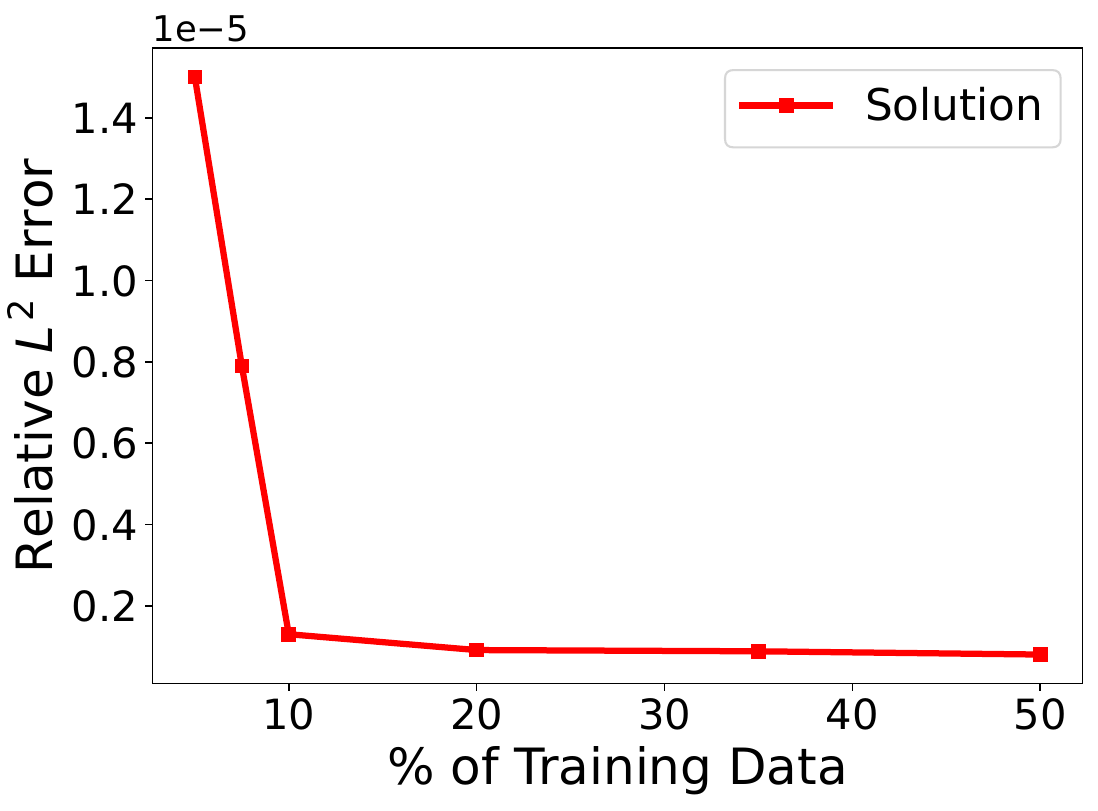}
        \caption{Boundary value problem}
        \label{fig:subfig_11_error_label}
    \end{subfigure}
        \begin{subfigure}[b]{0.3\textwidth}
        \centering
        \includegraphics[width=\textwidth]{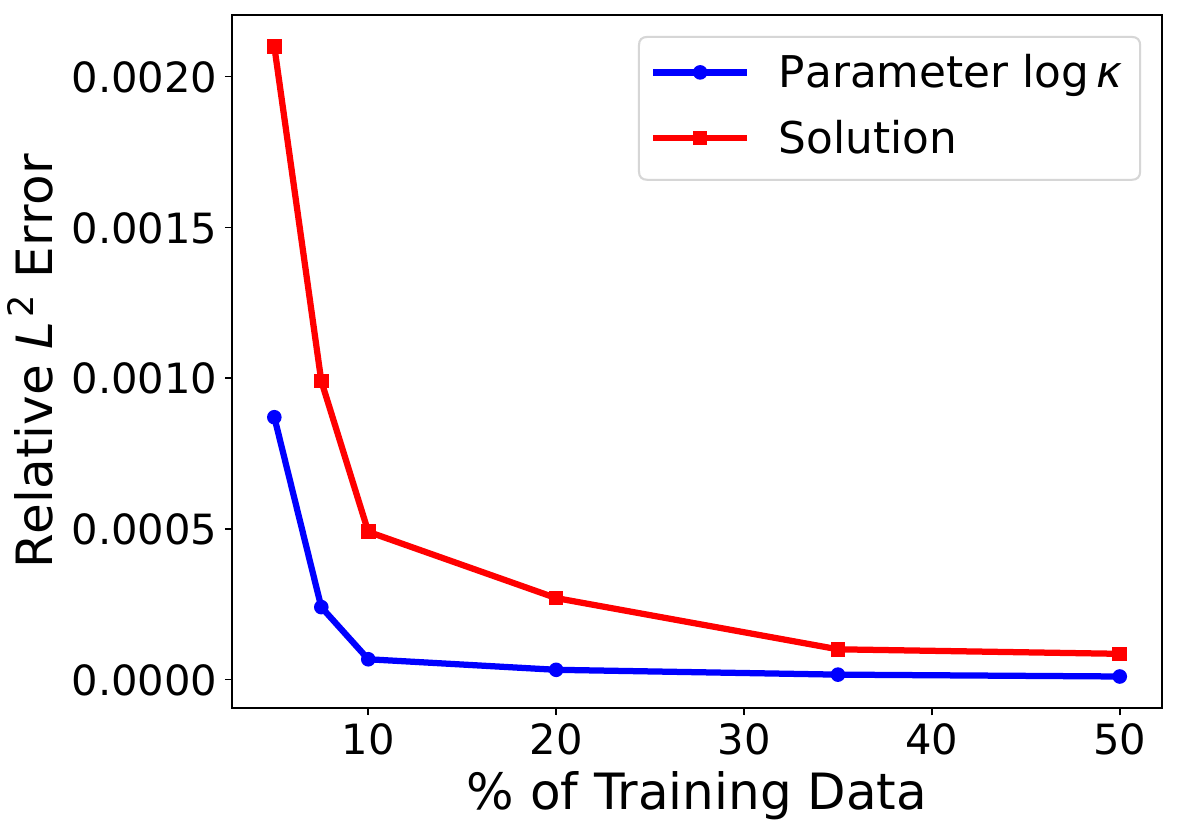}
        \caption{Darcy flow}
        \label{fig:subfig_33_error_label}
    \end{subfigure}
        \begin{subfigure}[b]{0.3\textwidth}
        \centering
        \includegraphics[width=\textwidth]{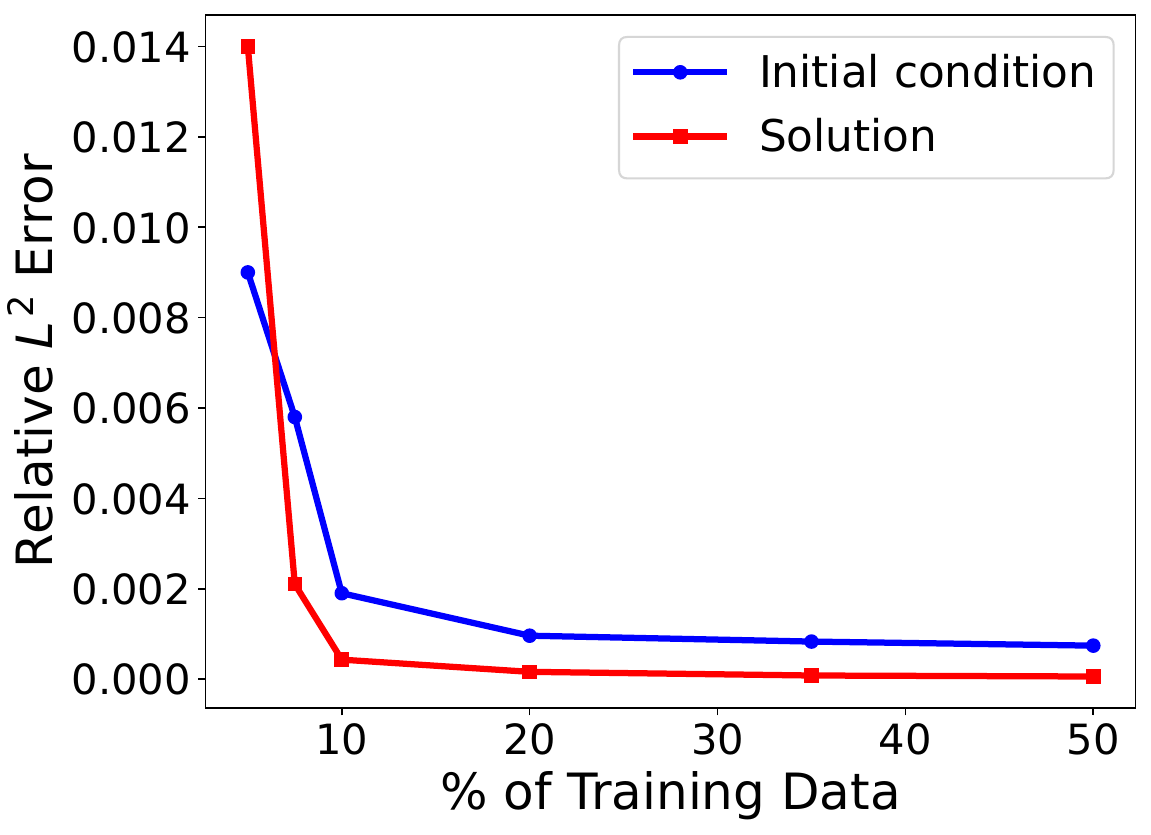}
        \caption{Navier Stokes}
        \label{fig:subfig_33_error_label_100}
    \end{subfigure}
    % Second subfigure
    \vfill
    \begin{subfigure}[b]{0.3\textwidth}
        \centering
        \includegraphics[width=\textwidth]{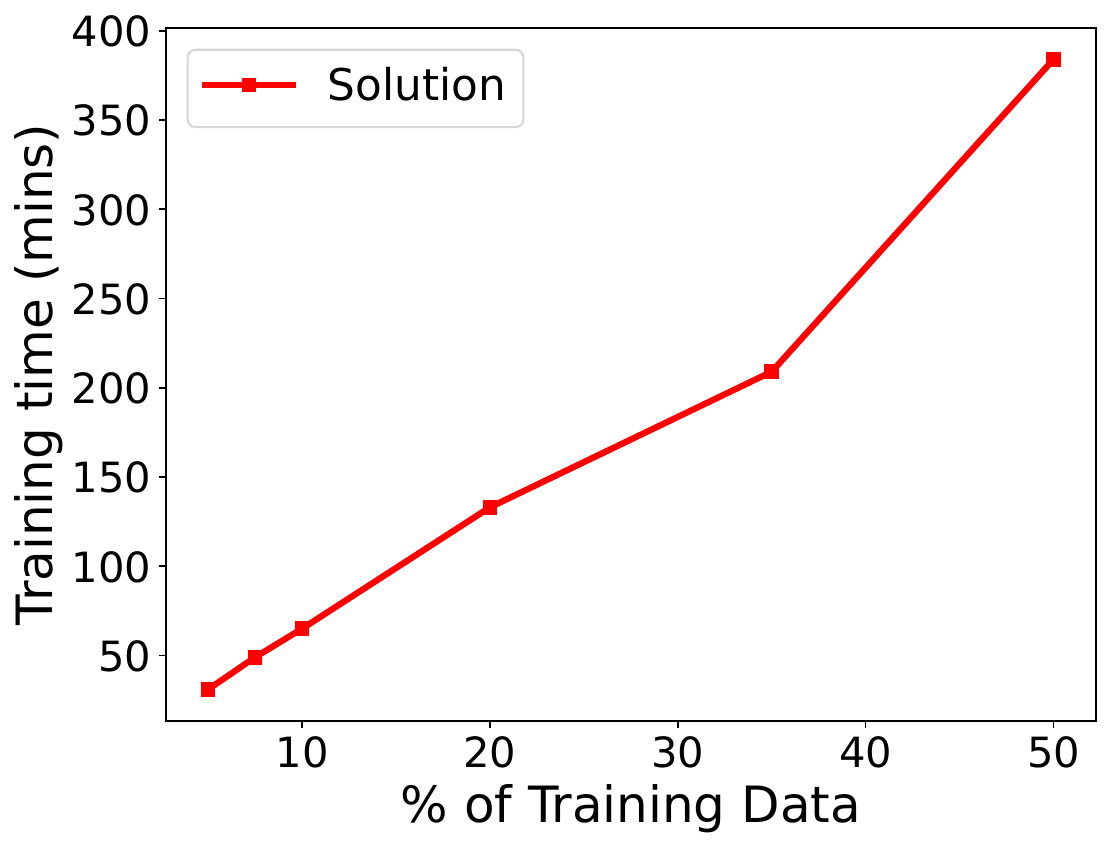}
        \caption{Boundary value problem}
        \label{fig:subfig_22_lppd}
    \end{subfigure}
    % Second subfigure
    \begin{subfigure}[b]{0.3\textwidth}
        \centering
        \includegraphics[width=\textwidth]{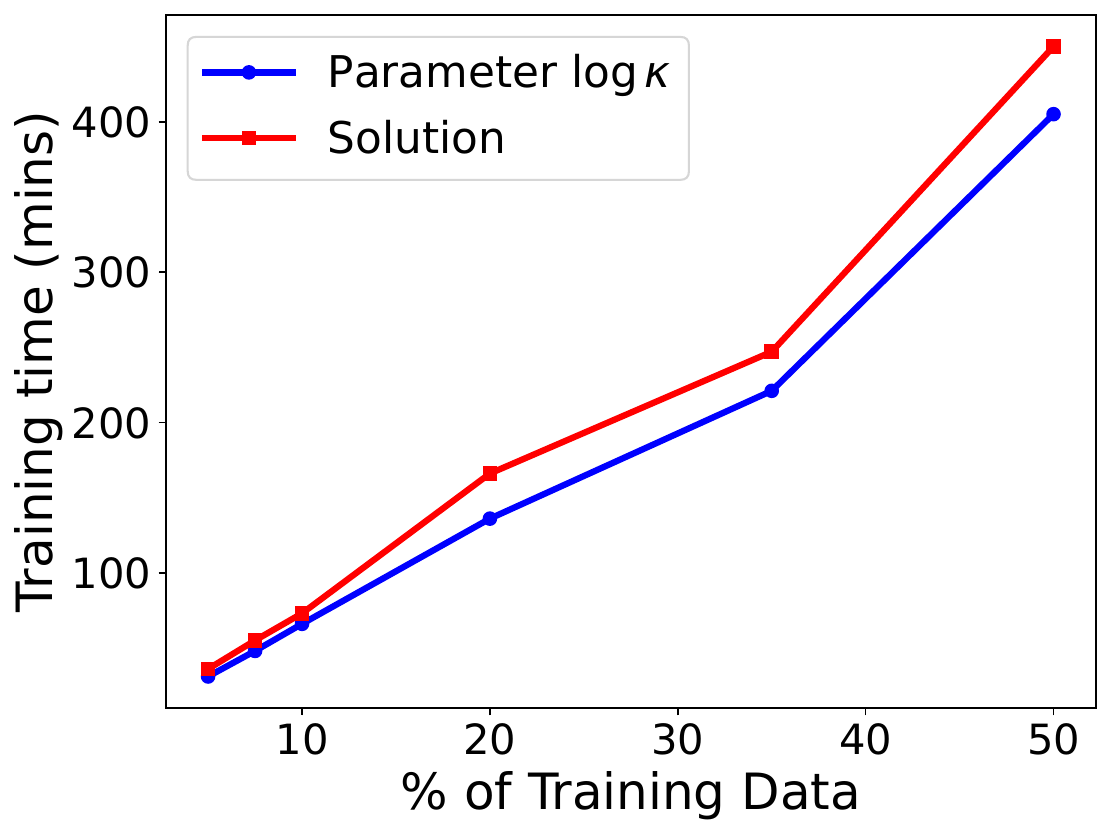}
        \caption{Darcy flow}
        \label{fig:subfig_44_lppd}
    \end{subfigure}
    \begin{subfigure}[b]{0.3\textwidth}
        \centering
        \includegraphics[width=\textwidth]{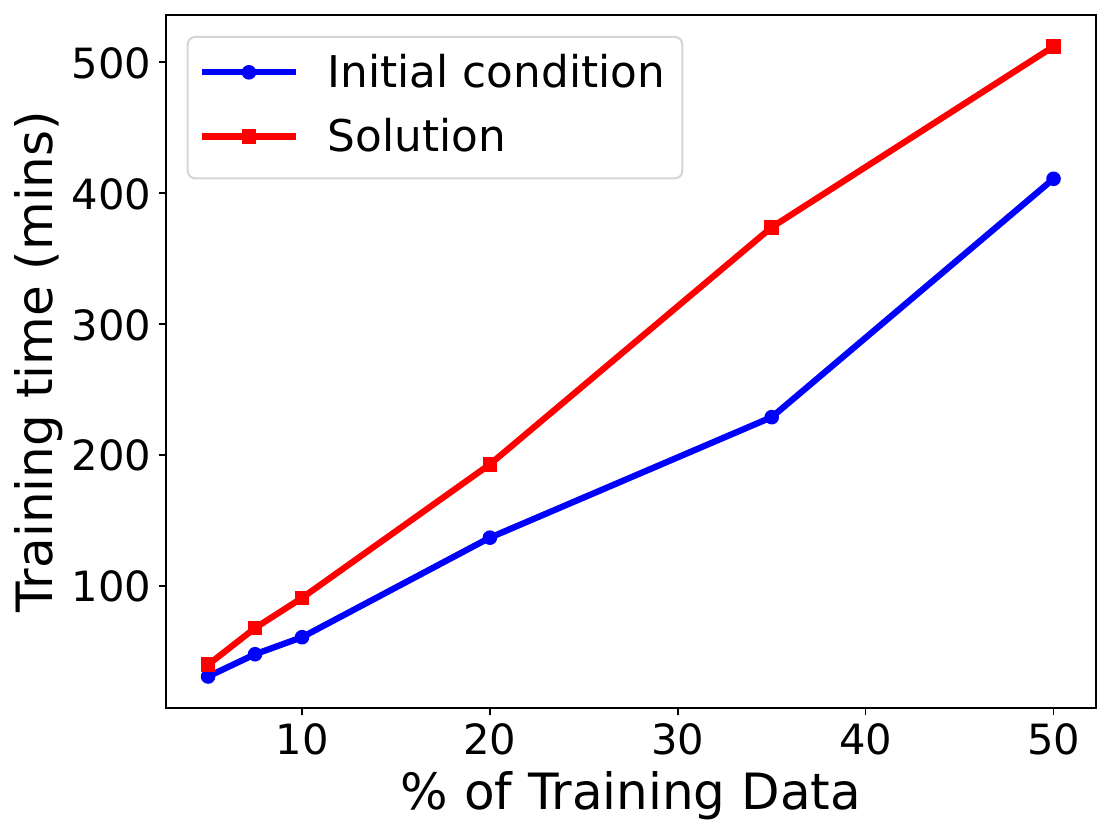}
        \caption{Navier Stokes}
        \label{fig:subfig_44_lppd_100}
    \end{subfigure}

    \caption{
    Evaluation of INR performance across different training data percentages.   
    Top: mean relative $L^2$ error norm for the INR reconstructed function, computed on a validation dataset of size 200. Bottom: training times of the INRs.}
    \label{fig:INR_analysis_performance}
\end{figure}

\begin{table}[h!]
  \caption{Final hyper-parameters for joint energy-based model in the different experiments.}
  \label{table-training-details-EBM}
  \begin{center}
  \begin{tabular}{ccccc}
      Dataset & $t$ & $M$ & $w$ & Epochs
   \\ \hline \\
     Boundary Value Problem &  1 & 4 & 1 & 1000\\
    Steady-State Diffusion & 0.5 & 16 & 1 & 1000\\
    Navier Stokes & 1 & 32 & 1 & 1000\\
    Lotka-Volterra Model & 1 & 4 & 1 & 1500\\
  \end{tabular}
  \end{center}
\end{table}

\subsection{Architectures}
\subsubsection{Implicit Neural Representation}
We have only made minor changes to the COIN++ model proposed by \citet{dupont2022coin}, so that it can take arbitrary point evaluations of the functions of interest. 
The number of layers and the dimension of each layer remain the same.
The initialisation scheme is the same as the one proposed in \citet{sitzmann2019siren}.

\subsubsection{Energy-Based Model}
In all our experiments, each training point for the joint energy-based model consists of the concatenation of the latent representations of the stochastic PDE solution and the functional or vector-valued parameter of the stochastic PDE, and our goal is to understand the connection between the two by learning their joint probability density.
To do this, we build on \citet{RefWorks:RefID:78-wang2022improved} to design the neural network architecture for the energy function. First, we uplift each part of the input vector into a latent space (using an encoder) so that they have the same dimension (equal to 128) and then propagate and merge them, with shared connections between the two branches. The architecture of the network is illustrated in Figure~\ref{fig:architecture}.

\begin{figure}[h]
  \centering
 \includegraphics[scale=0.75]{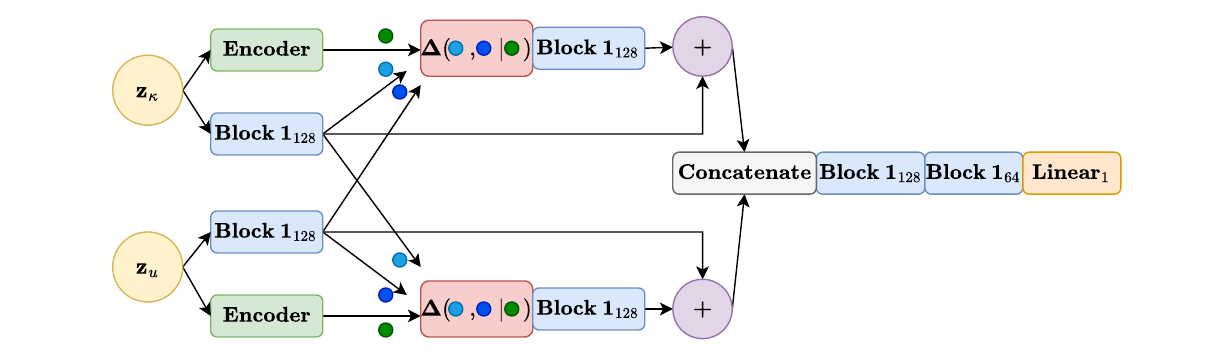}
  \caption{Energy-based model neural network architecture, where $\mathbf z_\kappa$ and $\mathbf z_u$ represent the finite dimensional latent embeddings of two functions, such as, the parameter and solution of the PDE.}
  \label{fig:architecture}
\end{figure}

The specific structure of each element of the architecture is the following
\begin{itemize}
    \item Encoder block
\begin{equation*}
    \text{Encoder}(\mathbf{x}) = \text{Linear} (\sigma \circ \text{Linear} (\boldsymbol{y}) + \boldsymbol{y}), \quad \boldsymbol{y} = \sigma \circ \text{Linear}(\mathbf{x}),
\end{equation*}
where $\sigma$ is a GELU \citep{hendrycks2016gelu} activation function.
\item Block 1
\begin{equation*}
     \text{Block 1}_{k}(\mathbf{x}) = \sigma \circ \text{Linear}(\mathbf{x}),
\end{equation*}
where $\sigma$ is a RELU activation and $k$ denotes the output dimension.
\item $\Delta$ operation
\begin{equation*}
    \Delta(\mathbf{x},\mathbf{y}|\mathbf{z}) = (1-\mathbf{z})\cdot \mathbf{x} + \mathbf{z}\cdot \mathbf{y},
\end{equation*}
where $\cdot$ denotes point-wise multiplication. Note that the vectors $\mathbf{x}, \mathbf{y}$ and $\mathbf{z}$ need to have the same dimension and the operation is just an interpolation.
\item $\boldsymbol{+}$ operation is a point-wise addition.
\end{itemize}
The dimension remains constant at 128 in the 2 branches of the architecture. Therefore, when both vectors are concatenated, it becomes 256. The last two Block 1's reduce the dimension from 256 to 128 and from 128 to 64, respectively.

\subsection{Optimal Bayesian Experimental Design for Sensor}
To estimate the utility function as presented in Appendix \ref{sec:sensor_placement_appendix} we use the following number of samples $N,L=50$ for all experiments. 
To obtain them, we sample from the prior, or updated posterior in subsequent iterations, using stochastic gradient Langevin dynamics for a total of 1000 steps, and take the last ones to calculate the utility.

\subsection{Choice of Hyperparameters}
Our method requires selecting the dimensions of the latent representations of the functional parameters and the solutions of the stochastic PDEs for the different numerical experiments. 
The latent dimension is selected based on the criterion that the mean relative $L^2$ error norm between the true function and the function reconstructed from the latent embedding using the INR remains below a certain threshold across all test data points. 
At the same time, we want the dimension to be low to ensure that the resulting distribution of the latent representations has a positive probability density that can be easily modelled by the energy-based model. 

Tables~\ref{table_bvp}, \ref{table_darcy}, \ref{table_navier_stokes} and \ref{table_lotka_volterra} show the mean relative $L^2$ error norms for different latent dimensions computed on a validation set of size 200 not seen during training for the solutions of the boundary value problem, the Darcy flow, the Navier Stokes equation, and the Lotka-Volterra model, respectively.

The final latent dimensions for the different numerical experiments are the following.
\paragraph{Boundary Value Problem} Recall that in this case the parameter is given by a two-dimensional real-valued vector. The latent dimension for the solution is 11.

\paragraph{Steady-State Diffusion} 
The latent dimensions for the diffusion coefficient and the solution are 32 and 16, respectively.

\paragraph{Navier Stokes Equation}
The latent dimensions for the initial vorticity and the vorticity at three different time snapshots $t=1,2 ,3$ are both 32.

\paragraph{Lotka Volterra Model} 
The latent dimension for both prey and predator trajectories is 11.

\begin{table}[h!]
\caption{Mean relative $L^2$ error norm for the INR reconstructed solution of the boundary value problem on a validation set of size 200.}
\label{table_bvp}
\begin{center}
\begin{tabular}{ccc}
\multicolumn{1}{c}{\bf Dimension}  &\multicolumn{1}{c}{$\Vert \widehat u-u_{\text{tr}}\Vert_{L^2}^2/\Vert u_{\text{tr}}\Vert_{L^2}^2$}
\\ \hline \\
15     & $6.4\times 10^{-7}$  \\
13      &$9.2\times 10^{-7}$ \\
11      &$1.3\times 10^{-6}$ \\
9      &$9.7\times 10^{-6}$ \\
7      &$8.6\times 10^{-5}$ \\
\end{tabular}
\end{center}
\end{table}

\begin{table}[h!]
\caption{Mean relative $L^2$ error norm for the INR reconstructed diffusion coefficient and solution of the Darcy flow equation on a validation set of size 200.}
\label{table_darcy}
\begin{center}
\begin{tabular}{ccc}
\multicolumn{1}{c}{\bf Dimension}  & $\Vert \log\widehat{\kappa}-\log\kappa_{\text{tr}}\Vert_{L^2}^2/\Vert \log\kappa_{\text{tr}}\Vert_{L^2}^2$ & \multicolumn{1}{c}{$\Vert \widehat u-u_{\text{tr}}\Vert_{L^2}^2/\Vert u_{\text{tr}}\Vert_{L^2}^2$}
\\ \hline \\
64 &  $5.4\times 10^{-6}$   & $7.1\times 10^{-8}$  \\
48 &  $8.3\times 10^{-5}$   & $6.5\times 10^{-6}$  \\
32 &  $6.7\times 10^{-5}$   &$5.8\times 10^{-5}$ \\
16  & $2.6\times 10^{-4}$   &$4.9\times 10^{-4}$ \\
8  &  $1.9\times 10^{-2}$  &$3.2\times 10^{-3}$ \\
\end{tabular}
\end{center}
\end{table}

\begin{table}[h!]
\caption{Mean relative $L^2$ error norm for the INR reconstructed initial vorticity and the vorticity at $t=1, 2, 3$, $\underline{w}_t$ of the Navier Stokes equation on a validation set of size 200.}
\label{table_navier_stokes}
\begin{center}
\begin{tabular}{ccc}
\multicolumn{1}{c}{\bf Dimension}   &  $\Vert \widehat{w}_0- w_{\text{tr},0}\Vert_{L^2}^2/\Vert w_{\text{tr},0}\Vert_{L^2}^2$   &$\Vert \widehat{\underline{w}}_{t}-\underline{w}_{\text{tr}, t}\Vert_{L^2}^2/\Vert \underline{w}_{\text{tr}, t} \Vert_{L^2}^2$
\\ \hline \\
64      & $3.5\times 10^{-4}$ &  $5.7\times 10^{-5}$\\
48      & $8.1\times 10^{-4}$ & $1.1\times 10^{-4}$ \\
32      & $1.9\times 10^{-3}$ & $4.3\times 10^{-4}$ \\
16      & $9.5 \times 10^{-3}$ & $9.2\times 10^{-4}$ \\
8      & $2.2\times 10^{-2}$ & $6.8\times 10^{-3}$ \\
\end{tabular}
\end{center}
\end{table}

\begin{table}[h!]
\caption{Mean relative $L^2$ error norm for the INR reconstructed prey and predator trajectories on a validation set of size 200.}
\label{table_lotka_volterra}
\begin{center}
\begin{tabular}{ccc}
\multicolumn{1}{c}{\bf Dimension}  &$\Vert \widehat{X}_1- X_{\text{tr},1}\Vert_{L^2}^2/\Vert X_{\text{tr},1}\Vert_{L^2}^2$   &$\Vert \widehat X_{2}-X_{\text{tr}, 2}\Vert_{L^2}^2/\Vert X_{\text{tr}, 2} \Vert_{L^2}^2$
\\ \hline \\
15      & $8.5\times 10^{-7}$  &  $7.8\times 10^{-7}$ \\
13      & $1.1\times 10^{-6}$ & $9.9\times 10^{-7}$ \\
11      & $3.3 \times 10^{-6}$ & $3.0\times 10^{-6}$ \\
9      & $9.9\times 10^{-6}$ & $8.4\times 10^{-6}$ \\
7      & $1.0\times 10^{-4}$ & $7.1\times 10^{-5}$ \\
\end{tabular}
\end{center}
\end{table}

\subsection{Times}
We present the data generation, training and inference times to assess the efficiency of our framework compared to classical approaches that use PDE solves within the MCMC. 
For data generation, we first obtain 10\% of the training dataset, which is the amount needed to train the INRs, while the remaining data can be generated in parallel with the training of the INRs for the functional parameter and the stochastic PDE solution. Additionally, the training of both INRs can also be performed in parallel.

After training the INRs we need to compute the latent codes $\mathbf{z}$ on the whole dataset, however this time is around a minute and therefore insignificant compared to the other times reported in Table \ref{table-times-ns}.  
The inference times correspond to the average time required to obtain an approximation of the posterior distribution (conditioned on the final number of observations used in the sensor placement task) by sampling 1000 iterations per chain for two independent chains using stochastic gradient Langevin dynamics. 
We note that the inference times are negligible compared to the training and data generation times, which are both performed off-line and only have to be done once. 
This makes our approach highly efficient for downstream tasks like sensor placement. 
For instance, in the case of the stochastic Navier Stokes equation, the adaptive Bayesian sensor placement loop requires about 20 minutes to identify fifteen new locations using our surrogate model. In contrast, conventional methods would take several days to achieve the same result, which exceeds the combined training and inference time of our approach.

\begin{table}[h!]
  \caption{Data generation, training and inference times for the different numerical examples.}
  \label{table-times-ns}
  \centering
  \begin{tabular}{llllll}
  \toprule
    Dataset   & \multicolumn{2}{c}{Data generation} & \multicolumn{2}{c}{Training} & Inference\\
    \cmidrule(r){2-3}
    \cmidrule(r){4-5}
       & 10\% & Rest & INR & EBM & \\
    \hline
    \\
    Boundary Value Problem & 1 mins & 9 mins & 65 mins& 34 mins& 1.1 mins\\
   \hline 
    Steady-State Diffusion & 8 mins &  70 mins& 73 mins& 40 mins& 2.5 mins\\
   \hline 
    Navier Stokes & 30 mins & 4 hours & 91 mins & 47 mins& 3.2 mins\\
    \hline 
    Lotka Volterra Model & 3 mins &  27 mins & 67 mins & 38 mins& 1.4 mins\\
    \bottomrule
  \end{tabular}
\end{table}

\subsection{Benchmarks}
In this section, we explain in more detail the FNO \citep{li2021fourier} and FNO with oracle noise \citep{salvi2022neural} baselines.
For all the numerical examples except for the Navier Stokes equation, the FNO with oracle noise reduces to an FNO with noise as an additional input, as the dynamics in these cases are temporally independent. 
The architecture and hyperparameters used for training the benchmarks are described below.

\paragraph{Boundary Value Problem} 
The FNO and the FNO with oracle noise are trained on 20000 epochs with a learning rate of $3\times 10^{-4}$. 
The hyperparameters used are as follows: 16 modes, 64 hidden channels and 4 Fourier layers.
It is worth noting that the FNO learns a mapping between function spaces, whereas our initial condition is represented by a two-dimensional real-valued vector. To address this, we learn a mapping between a linear interpolation of the boundary conditions and the solution.

\paragraph{Steady State Diffusion}
The FNO and the FNO with oracle noise are trained on 20000 epochs with a learning rate of $1\times 10^{-4}$. 
The hyperparameters used are as follows: 12 modes per dimension, 32 hidden channels and 4 Fourier layers.

\paragraph{Navier Stokes Equation}
The FNO is trained on 20000 epochs with a learning rate of $1\times 10^{-4}$. 
The hyperparameters used are as follows: 12 modes, 32 hidden channels and 4 Fourier layers.
For the FNO with oracle noise we use the same hyperparameters and training details specified in \citet{salvi2022neural}.

\paragraph{Lotka-Volterra Model}
In this example the neural operator baselines learn a mapping between the prey and predator trajectories.
Both baselines are trained on 15000 epochs with a learning rate of $5\times 10^{-3}$. 
The hyperparameters used are as follows: 16 modes, 64 hidden channels and 4 Fourier layers.

The neural operator benchmarks learn a deterministic mapping between inference parameters and the corresponding solution, in contrast to our neural coupling model, which learns a joint probability distribution over both parameters and solutions.

Once these baselines are trained on simulation data, they can be used for the sensor placement task. To achieve this, we need to be able to sample from the posterior distribution of the initial parameters, conditioned on the observations of the solution. In particular, the posterior distribution takes the form
\begin{align*}
    p(\kappa| \mathcal{D}) &\propto \prod_{i = 1}^M p(y_{i}|\kappa, \mathbf{x}_i)p(\kappa)  = \prod_{i = 1}^M \mathcal{N}(y_{i}|\textbf{NO}(\kappa)(\mathbf{x}_i), \sigma^2) p(\kappa),
\end{align*}
where \textbf{NO} denotes the neural operator surrogate which takes $\kappa$ as input and $p(\kappa)$ denotes the prior distribution over $\kappa$.

\section{Additional Results}
In the numerical experiments presented in the main text, we assume a fixed computational budget, meaning that the number of sensor locations for observations used in inference is fixed in advance.
However, it is interesting to analyse how the posterior contracts to the true value as new observations are included. To investigate this, we track the evolution of the relative $L^2$ error for the posterior mean for an increasing number of observations, comparing those taken from a quasi-random sequence with those obtained through adaptive optimisation of the locations. Results for the Darcy flow and Navier Stokes equation are provided in Figure \ref{fig:contraction_posterior}. We observe that the errors decrease faster for optimally selected points, chosen in an adaptive way, versus randomly chosen locations.

\begin{figure}[h]
    \centering

    % First subfigure
    \begin{subfigure}[b]{0.4\textwidth}
        \centering
        \includegraphics[width=\textwidth]{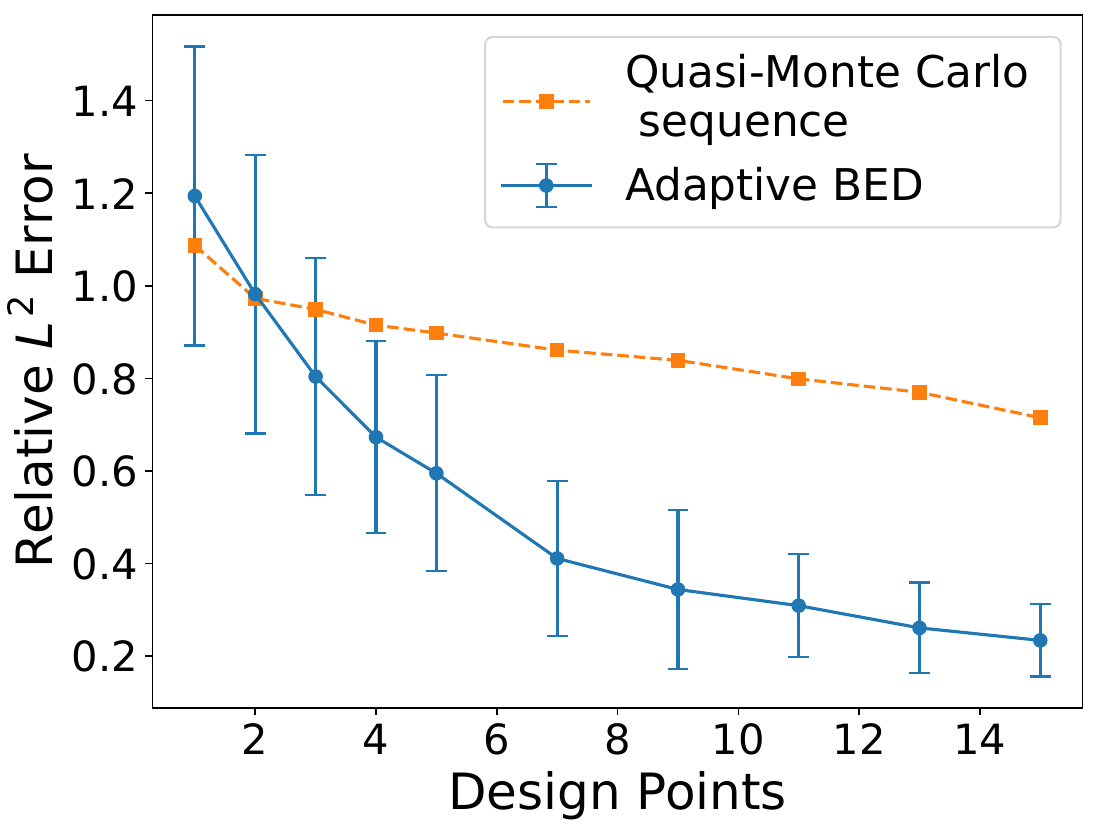}
        \caption{Darcy flow. Diffusion coefficient $\log \kappa$}
        \label{fig:subfig_1}
    \end{subfigure}
        \begin{subfigure}[b]{0.4\textwidth}
        \centering
        \includegraphics[width=\textwidth]{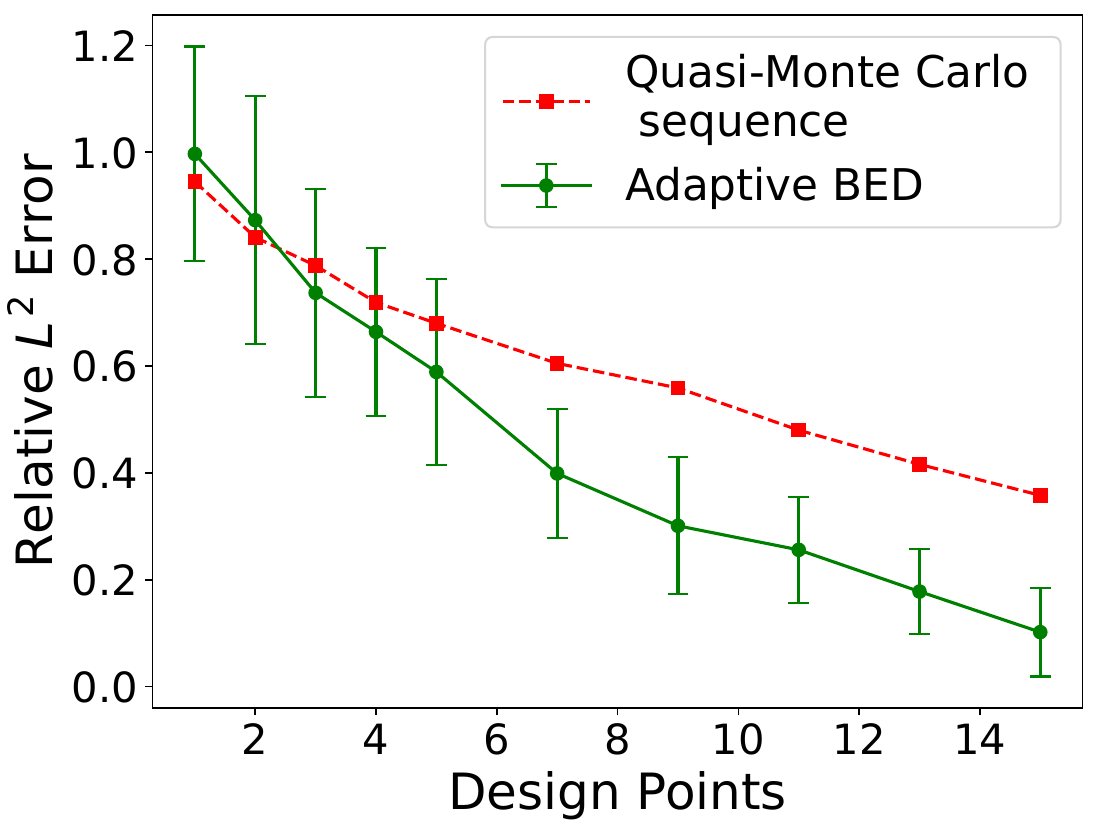}
        \caption{Darcy flow. Solution $u$}
        \label{fig:subfig_2}
    \end{subfigure}
    % Second subfigure
    \vfill
    \begin{subfigure}[b]{0.4\textwidth}
        \centering
        \includegraphics[width=\textwidth]{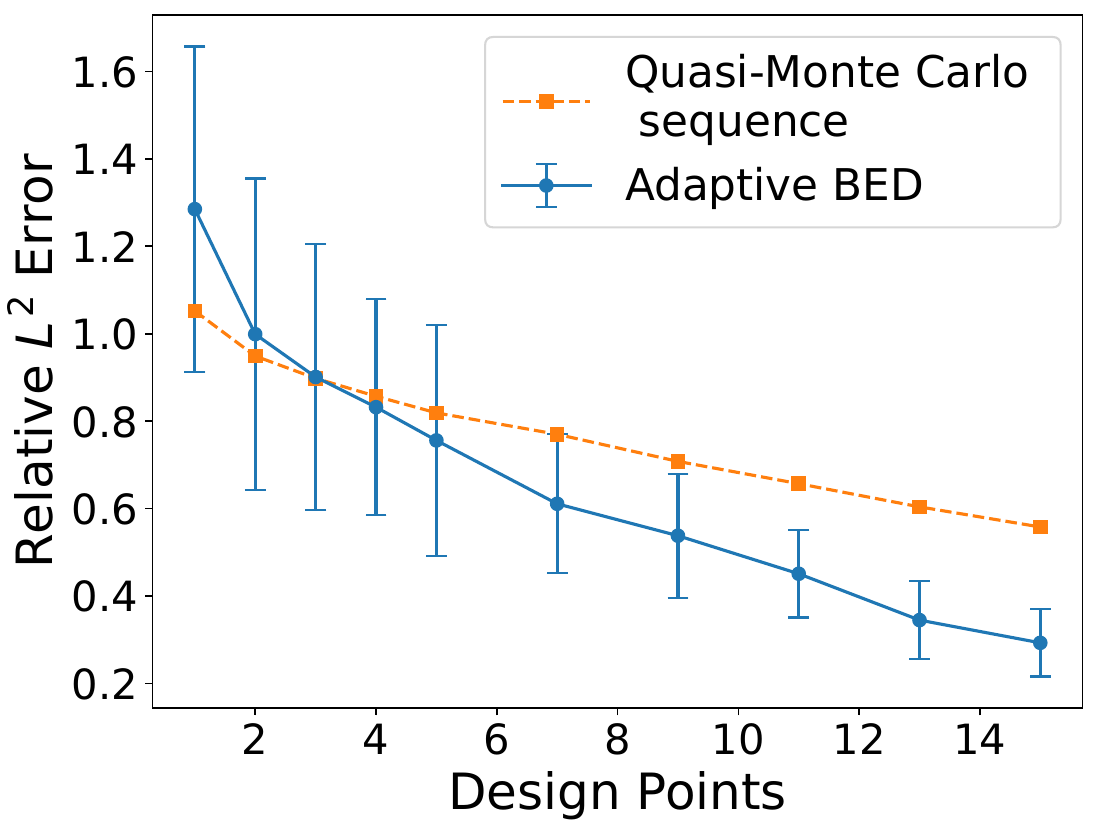}
        \caption{Navier Stokes. Initial vorticity $w_0$}
        \label{fig:subfig_3}
    \end{subfigure}
    % Second subfigure
    \begin{subfigure}[b]{0.4\textwidth}
        \centering
        \includegraphics[width=\textwidth]{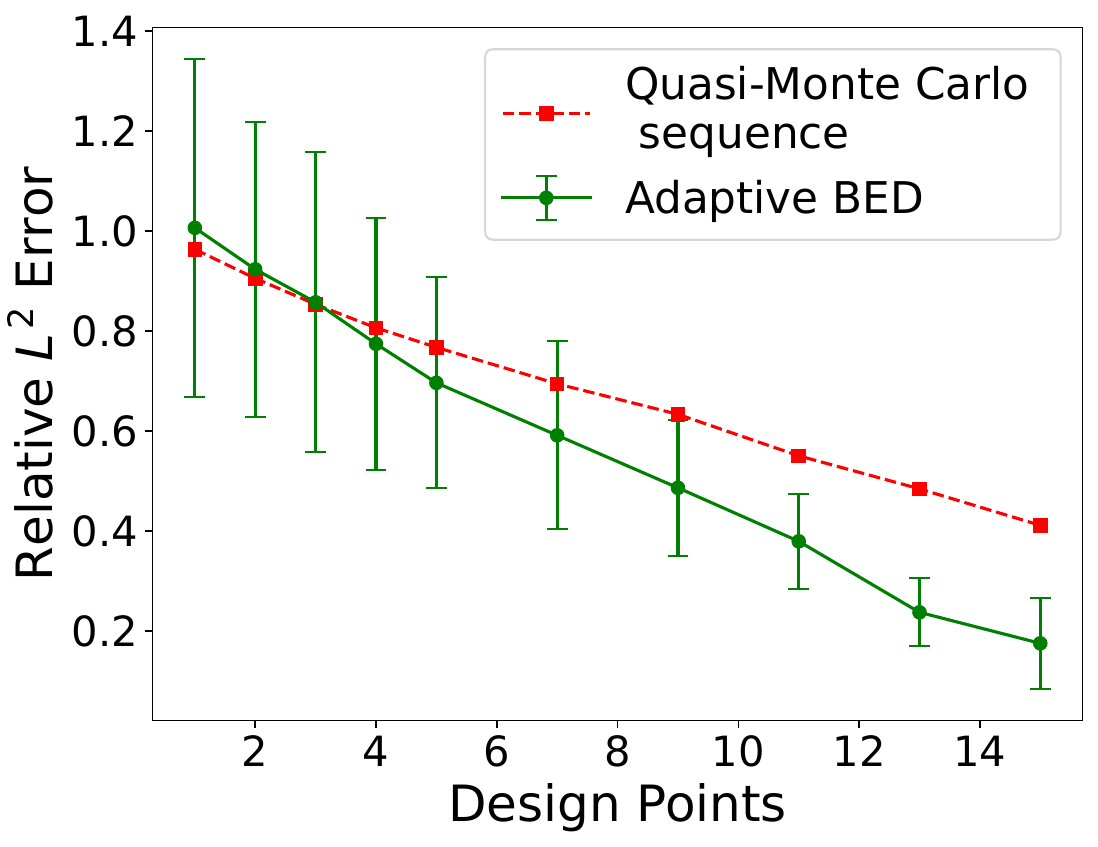}
        \caption{Navier Stokes. Vorticity at $t=1, 2, 3$, $\underline{w}_t$}
        \label{fig:subfig_4}
    \end{subfigure}

    \caption{
     Relative $L^2$ error norm for the posterior mean of the initial condition or parameter and the solution for the sensor placement experiment on the Darcy flow and Navier Stokes equations using an increasing number of observations. We take the average over 50 sensor placement loops.}
    \label{fig:contraction_posterior}
\end{figure}

\end{document}